\documentclass{article}

\PassOptionsToPackage{numbers,compress}{natbib}
\usepackage[main, final]{neurips_2026}

\usepackage[utf8]{inputenc} % allow utf-8 input
\usepackage[T1]{fontenc}    % use 8-bit T1 fonts
\usepackage[table]{xcolor}         % colors
\usepackage[pagebackref=true,breaklinks=true,letterpaper=true]{hyperref}       % hyperlinks
\definecolor{perfblue}{RGB}{64, 114, 175}
\definecolor{myred}{RGB}{194, 99, 93}
\hypersetup{
    colorlinks = true,
    citecolor= perfblue,
    linkcolor=myorange
}
\usepackage{url}            % simple URL typesetting
\usepackage{booktabs}       % professional-quality tables
\usepackage{amsfonts}       % blackboard math symbols
\usepackage{nicefrac}       % compact symbols for 1/2, etc.
\usepackage{microtype}      % microtypography

\usepackage{subcaption}
\usepackage{wrapfig}

\usepackage{csquotes}
\usepackage{comment}
\usepackage[page]{appendix}
\usepackage{amsthm}
\usepackage{mathtools}
\usepackage{algorithm}
\usepackage{multicol}
\usepackage{algcompatible}

\renewcommand{\appendixtocname}{Contents of appendices}

\usepackage{array, multirow,graphicx}
\usepackage{float}
\usepackage{pifont}% http://ctan.org/pkg/pifont
\usepackage{amsmath}

\definecolor{darkbrown}{rgb}{0.7,0.2,0.1}

\definecolor{orange}{rgb}{1,0.5,0}

\definecolor{darkgreen}{rgb}{0,0.5,0}
\definecolor{grey}{rgb}{0.7,0.7,0.7}

\definecolor{mypurple}{rgb}{0.56,0.3059,0.54}
\definecolor{myyellow}{rgb}{0.98,0.69,0.282}
\definecolor{mypink}{RGB}{242,196,195}
\definecolor{myblue}{rgb}{0.2745,0.3255,0.5647}
\definecolor{myorange}{rgb}{0.7,0.3647,0.22745}
\definecolor{mydarkred}{rgb}{0.5686,0.1333,0.141}
\definecolor{mygreen}{rgb}{0.3255,0.596,0.3}

\newcommand{\myblue}{\color{myblue}}
\newcommand{\myred}{\color{myred}}

\newcommand{\mygreen}{\color{mygreen}}
\newcommand{\myorange}{\color{myorange}}
\definecolor{colorpo}{rgb}{0.87,0.87,0.9}

\definecolor{colorro}{rgb}{0.9569,0.9451,0.932}

\definecolor{blue1}{rgb}{0.4,0.757,0.887}
\definecolor{blue2}{rgb}{0.63,0.84,0.91}
\definecolor{blue3}{rgb}{0.89,0.95,0.97}
\usepackage{diagbox}

\theoremstyle{plain}
\newtheorem{theorem}{Theorem}%[section]
\newtheorem{proposition}[theorem]{Proposition}

\theoremstyle{definition}

\theoremstyle{remark}

\usepackage{bm}

\usepackage{xspace}

\usepackage{tabularx}

\makeatletter
\let\oldappendix\appendices

\renewcommand{\appendices}{%
  \clearpage
  \renewcommand{\thesection}{\Roman{section}}
  \let\tf@toc\tf@app
  \addtocontents{app}{\protect\setcounter{tocdepth}{3}}
  \immediate\write\@auxout{%
    \string\let\string\tf@toc\string\tf@app^^J
  }
  \oldappendix
}%

\newcommand{\listofappendices}{%
  \begingroup
  \renewcommand{\contentsname}{\appendixtocname}
  \let\@oldstarttoc\@starttoc
  \def\@starttoc##1{\@oldstarttoc{app}}
  \tableofcontents% Reusing the code for \tableofcontents with different \contentsname and different file handle app
  \endgroup
}

\makeatother

\title{Loop-Free Inverse Reinforcement Learning via Sequential Value Recovery with Q-Score Matching}

\author{%
  Yang Chen$^{1}$\thanks{Equal contribution.} \quad
  Yitan Zhang$^{2}$\footnotemark[1] \quad
  {\bf Michael J.~Witbrock}$^{2}$ \quad
  {\bf Shuyue Hu}$^{1}$\thanks{Corresponding author.} \\
  $^{1}$ Shanghai Artificial Intelligence Laboratory \quad
  $^{2}$ University of Auckland \\
  \texttt{chenyang4@pjlab.org.cn, yzhb332@aucklanduni.ac.nz,}\\
\texttt{m.witbrock@auckland.ac.nz, hushuyue@pjlab.org.cn}
}

\begin{document}

\maketitle

\begin{abstract}
Inverse Reinforcement Learning (IRL) aims to recover a reward function that explains expert demonstrations. Existing IRL methods typically rely on a bi-level optimization procedure that alternates between reward learning and policy optimization, leading to substantial computational burden and training instability. In this work, we introduce a different route that eliminates policy optimization entirely by leveraging diffusion policies. Our key insight is that a diffusion policy encodes the action-gradient structure of the optimal soft Q function, enabling reward learning to be cast as a sequence of value recovery problems, thereby allowing us to bypass reward-policy loops inherent in prior IRL methods. Specifically, our method proceeds in three stages: (I) recovering the optimal soft Q function via action-gradient matching and estimating the corresponding soft value function (LogSumExp of Q values) in a way inspired by Gumbel regression; (II) calibrating these soft values by inferring a state-dependent offset; (III) extracting the reward by enforcing Bellman consistency. This leads to \emph{Loop-Free Inverse Reinforcement Learning} (LFIRL), a fully offline algorithm that operates in a simple, loop-free, and sequential manner. LFIRL is simple to implement and significantly improves training efficiency while maintaining strong reward recovery performance.  Empirically, across Maze, Franka Kitchen, Adroit Hand Pen, and Push-T benchmarks, LFIRL achieves \textbf{2-3x} speedup over the fastest baselines, while matching or surpassing state-of-the-art methods in reward recovery quality.~\footnote{The implementation is available at \url{https://github.com/NothingThrough/LFIRL}.}

%Inverse Reinforcement Learning (IRL) aims to recover a reward function that explains expert demonstrations. Existing IRL methods typically have an inherent bi-level optimization structure that repeatedly executes the update {\em loops}  of the reward and the policy, which often introduces computational burden and learning instability. In this work, we introduce a different route for IRL, which recovers a reward function from expert demonstrations {\em without} policy optimization processes, drawing inspiration from the diffusion policy. Our key insight is that a diffusion policy encodes the action-gradient structure of the optimal soft Q function, which can turn reward learning into a three-stage value-recovery problems: (I) first estimate an optimal soft Q function via action gradient matching and the corresponding soft value function (LogSumExp of Q values) inspired by Gumbel regression; (II) then calibrate them by inferring a state-dependent offset; (III) finally extract the reward via enforcing Bellman consistency. This insight leads to our {\em Sequential Inverse Reinforcement Learning} (LFIRL) algorithm; it is fully {\em offline} that runs in a fully {\em sequential} manner.  This algorithm is practically simple and significantly improves the training efficiency, while preserving strong reward recovery performance. Empirically, on Maze, Franka Kitchen, AdroitHandPen, and Push-T benchmarks, LFIRL improves training efficiency by \textbf{2-3x} compared with the fastest baselines, while matching or surpassing state-of-the-art baselines in reward recovery.    
\end{abstract}

\section{Introduction}\label{sec:intro}

Inverse reinforcement learning (IRL)  \citep{ng2000algorithms,abbeel2004apprenticeship} studies how to recover a reward function from expert demonstrations, with the goal of explaining expert behavior through an underlying optimization principle. %This perspective is especially appealing when expert behavior is available but the desired objective is difficult to engineer manually, and 
It has become an important paradigm for reward design, behavior understanding, and transferable decision making \citep{kalakrishnan2013learning,sharifzadeh2016learning}. In practice, however, a major bottleneck in IRL is time efficiency due to the inherent  \emph{loop structure} underlying many classic and modern approaches \citep{abbeel2004apprenticeship,ziebart2008maximum,finn2016guided,fu2018learning,zeng2022maximum}: the alternating update of the reward function and the policy. This loop arises because most IRL methods generally cannot extract reward-supervision signals directly from demonstrations. Instead, the reward is usually learned through differentiating expert behavior from the learner behavior. This mutual dependence naturally leads to a classic {\em bi-level} optimization process, where a reward function is updated by contrasting expert demonstrations with learner behavior, the learner policy is then updated under the updated reward, and this process is repeated until convergence. Such a looped training paradigm is computationally expensive, especially in high-dimensional tasks, and also makes the overall training process harder to stabilize. This raises our central question: \emph{can we design an IRL method without the reward-policy loop while still recovering a useful reward function?}

\begin{figure}[t]
    \centering
    \includegraphics[width=\linewidth]{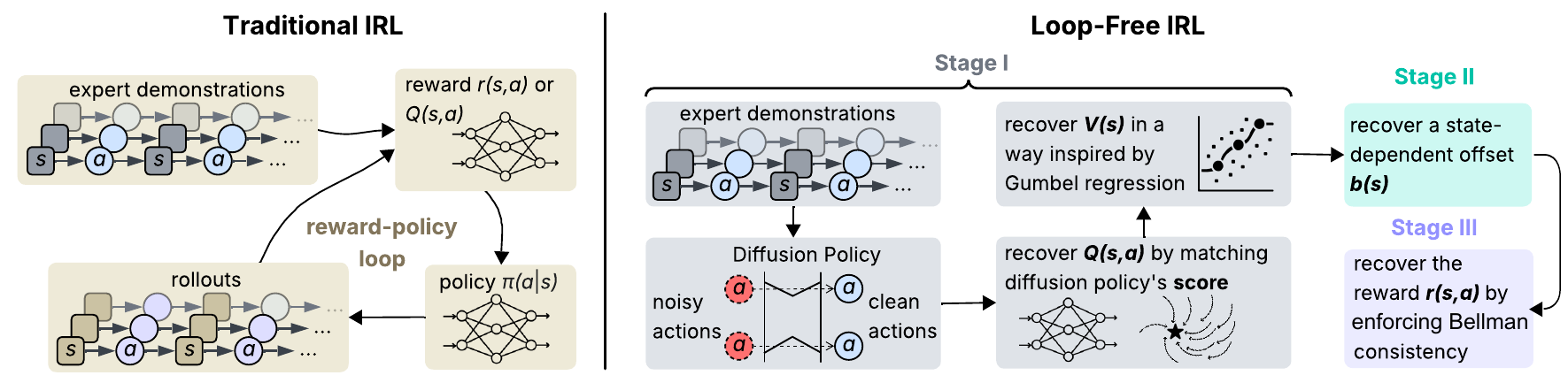}
    %\vspace{-1em}
    \caption{\small Traditional inverse reinforcement learning alternates between reward learning and policy optimization, while our loop-free version trains each component once without a reward-policy loop.} 
    \label{fig:irl_comparison}
    %\vspace{-2em}
\end{figure}

If we could extract reward-supervision signals directly from demonstrations, then reward recovery would no longer need to be mediated by a reward-policy loop. Intuitively, demonstrations already contain local information about which nearby actions are more or less preferable around expert behavior. {\em Diffusion policy} methods \citep{chi2025diffusion,psenkalearning,reuss2023goal} provide a natural way to represent this information. Built on the recently burgeoning diffusion models \citep{ho2020denoising,songscore}, a diffusion policy represents expert  action distribution by transforming noise into actions through an iterative refinement process: %Due to its strong generative modeling capability, expressive policy representation and ability to model multimodal action distributions \citep{janner2022planning}, it has achieved strong empirical performance in imitation learning \citep{wudiffusing} and reinforcement learning \citep{venkatraman2023reasoning,wang2022diffusion,zhu2023diffusion}. %and recently IRL \citep{huang2024diffusion,lai2024diffusion,yoon2024maximum}. 
it learns the gradient of the action distribution's {\em score function}, $\nabla_{\mathbf{a}} \log \pi^*(\mathbf{a}| \mathbf{s})$, that describes the direction in which actions should move towards expert ones, and iteratively optimizes wrt this gradient field via a series of stochastic steps. This score-based signal is exactly what we need to break the reward-policy loop. Specifically, in the commonly used maximum-entropy reinforcement learning \citep{haarnoja2018soft} setting, %a pretrained diffusion policy provides a score field that captures the expert action distribution, and under the maximum-entropy principle \citep{haarnoja2018soft}, 
this score function encodes information of the gradient of the expert's optimal soft Q function wrt actions, i.e., $\nabla_{\mathbf{a}} \log \pi^*(\mathbf{a}| \mathbf{s}) = \nabla_{\mathbf{a}} Q^*(\mathbf{a}, \mathbf{s})$. Therefore, by matching its action gradient, we can estimate an optimal soft Q function directly from demonstrations, without relying on the differential information from an iteratively updated learner policy. %Fig.~\ref{fig:gradient} illustrates this connection between diffusion-policy score fields and action-gradient matching.%

%\begin{wrapfigure}{r}{.4\textwidth}
%\begin{minipage}[t]{0.4\textwidth}
%    \centering
%    \includegraphics[height=4.4cm]{figs/gradient.pdf}
%    \captionof{figure}{\small In Push-T, the goal is to move the {\grey T-shaped block} to the {\mygreen target pose}. For each observation, the gradient of the soft $Q$-function with respect to the action points toward the locally optimal direction (black arrows). Moreover, the gradient magnitude decreases along the optimal action path and approaches zero near the optimum.}
%    \label{fig:gradient}
%\end{minipage}
%\end{wrapfigure}

Based on this observation, we propose a new IRL algorithm that 
%\textbf{Sequential Inverse Reinforcement Learning (LFIRL)}, which is the {\em first} IRL method that 
avoids the conventional reward-policy loops and instead recovers the reward in a {\em fully sequential} manner (see Fig.~\ref{fig:irl_comparison}). %in order to promise the time efficiency. 
Concretely, it proceeds in three sequential stages:  
\textbf{(I)} This stage first recovers an optimal soft Q function by matching its action gradients to the score induced by a diffusion policy, where we use a Sobolev-style training mechanism \citep{czarnecki2017sobolev} to favor expert actions over nearby perturbations; then estimates the corresponding soft value function (V $=$ LogSumExp Q) in a way inspired by Gumbel regression \citep{gargextreme} (Sec.~\ref{sec:qv}). 
\textbf{(II)} For deriving calibrated value functions, the second stage infers a state-dependent offset function, appending which to the already estimated Q and V yields their more accurate estimates (Sec.~\ref{sec:offset}). %producing a state-value estimate consistent with the recovered action-value structure . 
\textbf{(III)} The final stage recovers the reward function from the calibrated Q and V functions by enforcing Bellman consistency (Sec.~\ref{sec:r}). The entire pipeline is {\em fully offline}, using only a static set of trajectories. %Empirically, we evaluate the algorithm on Maze \citep{gymnasium_robotics2023github}, Franka Kitchen \citep{gymnasium_robotics2023github}, and Push-T \citep{chi2025diffusion}. 
We empirically demonstrate that 
%Across these benchmarks, 
our resulting algorithm matches or surpasses state-of-the-art IRL baselines in reward recovery while significantly reducing training time (Sec.~\ref{sec:experiment}). % by \textbf{2-5x} compared with the fastest baseline in most settings (Sec.~\ref{sec:experiment}).

Our contributions are summarized as follows:
\begin{enumerate}
    \item We introduce diffusion policy to IRL and propose the \textbf{Loop-Free Inverse Reinforcement Learning  (LFIRL)} algorithm. It is a \textbf{fully offline} IRL algorithm that {\bf avoids the computationally expensive looped reward-policy updates} in conventional IRL methods by leveraging the diffusion policy to realize a sequential value recovery procedure.
    \item We conduct empirical studies on tasks including Maze, Franka Kitchen, Adroit Hand Pen, and Push-T, showing that our method matches or surpasses strong  online and offline baselines in reward recovery while \textbf{improving training time efficiency by 2-3x} in most settings. 
\end{enumerate}

\section{Related Work}\label{sec:literature}

\textbf{Inverse reinforcement learning.} IRL has been studied since the early formulations of reward recovery and apprenticeship learning \citep{abbeel2004apprenticeship,ng2000algorithms,ratliff2006maximum,ziebart2008maximum,ziebart2010modeling}, and was later extended to deep-learning variants such as Guided Cost Learning \citep{finn2016guided}. More recent adversarial formulations remain a dominant route for IRL and imitation learning, but they typically rely on a {\em nested loop} between reward/discriminator learning and policy optimization, including GAIL \citep{ho2016generative}, AIRL \citep{fu2018learning}, DAC \citep{kostrikovdiscriminator}, and observation-only adversarial variants such as GAIfO \citep{torabi2018generative}. A broader unifying view is provided by the moment-matching perspective of imitation learning \citep{swamy2021moments}, while empirical studies such as \citep{orsini2021matters} further highlight the sensitivity of adversarial pipelines to design choices. Several recent methods reduce or modify this classical reward-policy loop: Maximum-Likelihood IRL \citep{zeng2022maximum} proposes a {\em single-loop} update, ValueDICE \citep{kostrikov2019imitation} avoids a separate RL optimization procedure through off-policy distribution matching, IQ-Learn \citep{garg2021iq} and LS-IQ \citep{al2023ls} learn implicit Q-function formulations, PIRO \citep{chen2025trust} stabilizes non-adversarial reward-policy learning through trust-region reward optimization. Other approaches modify the imitation or policy-search process in different ways, including Coherent Soft Imitation Learning \citep{watson2023coherent}, successor-feature matching \citep{jainnon}, FILTER \citep{swamy2023inverse}, and Hybrid IRL \citep{ren2024hybrid}. In contrast to both nested-loop methods and single-loop approaches that tightly entangle reward and policy learning, our method adopts a {\em fully staged, sequential} IRL pipeline: it uses diffusion-derived local action-score information to recover soft values, and then recovers rewards.

\textbf{Diffusion-Based Imitation Learning.} Diffusion models have rapidly become strong policy classes for imitation learning, especially in robotics and multimodal control. Diffusion Policy \citep{chi2025diffusion} models actions through a conditional denoising process and has become a standard baseline for visuomotor behavior cloning, while subsequent work extends diffusion-based imitation learning through goal-conditioned policies \citep{reuss2023goal}, diffusion-augmented behavioral cloning \citep{chendiffusion}, improved visual or self-supervised representations \citep{li2024crossway,Ze2024DP3}, and architectural scaling studies \citep{dasari2025ingredients,zhu2025scaling}. A smaller line of work uses diffusion models more directly for reward learning, for example by extracting reward-like quantities from diffusion models \citep{nuti2023extracting}, learning rewards from expert videos \citep{Huang2023DiffusionReward}, leveraging score information for policy optimization \citep{psenkalearning}, or integrating diffusion into adversarial imitation pipelines as the {\em discriminator} such as DiffAIL \citep{wang2024diffail}, DRAIL \citep{lai2024diffusion}, and DIFO \citep{huang2024diffusion}. Recent score-matching-based imitation methods such as \citep{wudiffusing} likewise emphasize the value of diffusion-derived gradient information. Our method differs from these directions in that we use a diffusion policy only as a score estimator for local supervision of soft $Q$-recovery, and then recover $V$ and $r$ sequentially, rather than using diffusion as the final policy class, extracting relative trajectory-level rewards from multiple diffusion models, or embedding diffusion inside an adversarial reward-policy loop.

\section{Preliminaries}\label{sec:pre}

A Markov decision process (MDP) is defined by the tuple $(\mathcal{S},\mathcal{S}_\bot, \mathcal{A},P,\gamma,r)$, where $\mathcal{S}$ and $\mathcal{A}$ denote the state and action spaces, $\mathcal{S}_\bot \subseteq \mathcal{S}$ is a set of absorbing states, $P(\cdot | \mathbf{s},\mathbf{a})$ is the transition kernel, $\gamma\in (0,1)$ is the discount factor, and $r:\mathcal{S}\times\mathcal{A}\to\mathbb{R}$ is the reward function. %and $T>0$ is the time horizon. 
For any absorbing state $\mathbf{s}_\bot\in \mathcal{S}_\bot$, we have $P(\mathbf{s}_\bot | \mathbf{s}_\bot, \mathbf{a}) = 1$ and $r(\mathbf{s}_\bot,\mathbf{a}) = 0$ for all $\mathbf{a}\in\mathcal{A}$, which describes can-never-escape states, e.g., completions or failures of a game. 
Let $\pi(\mathbf{a}| \mathbf{s})$ be a stochastic policy, and let $\rho^\pi(\mathbf{s},\mathbf{a}) \coloneqq \frac{1}{1-\gamma} \sum_{t=0}^\infty  \Pr(\mathbf{s}_t = \mathbf{s}, \mathbf{a}_t = \mathbf{a} | \pi, P) $ 
denote its induced state-action density.

\subsection{Maximum-Entropy RL and IRL}
The goal of classic RL is to find a policy to maximize the expected discounted long-term rewards $\mathbb{E}_{(\mathbf{s},\mathbf{a})\sim\rho^\pi}[r(\mathbf{s},\mathbf{a})]$. 
We consider a generalized version of 
Maximum-entropy (MaxEnt) RL that augments the reward objective with the relative entropy $\mathcal{H}_\mu(\pi)
\coloneqq
\mathbb{E}_{(\mathbf{s},\mathbf{a})\sim\rho^\pi}
[
-\log \tfrac{\pi(\mathbf{a} | \mathbf{s})}{\mu(\mathbf{a} | \mathbf{s})}
]$ between $\pi$ and a reference policy $\mu$: 
$
\mathbb{E}_{(\mathbf{s},\mathbf{a})\sim\rho^\pi}[r(\mathbf{s},\mathbf{a})]
+ \varepsilon \, \mathcal{H}_\mu(\pi)
$, where $\varepsilon>0$ is the temperature parameter.
It recovers the standard MaxEnt RL \cite{haarnoja2018soft} objective up to a constant when $\mu$ is uniform.

%\cite{haarnoja2018soft} that aims to learn a policy that maximizes the expected reward together with an entropy regularizer:
%\begin{equation}
%\label{eq:maxent_rl}
%$J(\pi;r)
%=
%$\mathbb{E}_{(\mathbf{s},\mathbf{a})\sim\rho^\pi}[r(\mathbf{s},\mathbf{a})] + \varepsilon \, \mathcal{H}(\pi)$, 
%\end{equation}
%where $\varepsilon>0$ is the temperature parameter, and $\mathcal{H}(\pi)
%\coloneqq
%\mathbb{E}_{(\mathbf{s},\mathbf{a})\sim\rho^\pi}
%\big[-\log \pi(\mathbf{a} | \mathbf{s})\big]$ denotes the causal entropy \citep{ziebart2010modeling}. 
Under this formulation of MaxEnt RL, the optimal policy admits the energy-based form
\begin{equation}
\label{eq:boltzmann_mu}
\pi^*(\mathbf{a} | \mathbf{s})=\mu(\mathbf{a} | \mathbf{s})
\exp\left( {\textstyle \frac{1}{\varepsilon}} \big(Q^*(\mathbf{s},\mathbf{a})-V^*(\mathbf{s})\big)\right),
\end{equation}
where the optimal soft value function $V^*(\mathbf{s})$ and optimal soft action value function $Q^*(\mathbf{s},\mathbf{a})$ satisfy
\begin{equation}
\label{eq:soft_q_mu}
V^*(\mathbf{s}) = \varepsilon \log \int_{\mathcal A}
\exp \left({\textstyle \frac{1}{\varepsilon}}Q^*(\mathbf{s},\mathbf{a})\right) d\mu(\mathbf{a}|\mathbf{s}),
\quad
Q^*(\mathbf{s},\mathbf{a}) = r(\mathbf{s},\mathbf{a}) +
\gamma \mathbb{E}_{\mathbf{s}'\sim P} \left[V^*(\mathbf{s}')\right].
\end{equation}
Note that the MaxEnt optimality condition in Eq.~(\ref{eq:boltzmann_mu}) implies the following relation, $\pi^*(\mathbf{a}|\mathbf{s})/\mu(\mathbf{a}|\mathbf{s})$ and $Q^*(\mathbf{a}|\mathbf{s})$ have the same gradient wrt actions, which we will leverage to build our method upon: 
\begin{equation}
\label{eq:q_derivative}
    \nabla_{\mathbf{a}} \log {\textstyle\frac{\pi^*(\mathbf{a}|\mathbf{s})}{\mu(\mathbf{a}|\mathbf{s})}}={\textstyle \frac{1}{\varepsilon}}\nabla_{\mathbf{a}} Q^*(\mathbf{s},\mathbf{a}).
    %\nabla_{\mathbf{a}} \log \pi^*(\mathbf{a} | \mathbf{s}) = \nabla_{\mathbf{a}} \log \mu(\mathbf{a}|\mathbf{s}) + {\textstyle \frac{1}{\varepsilon}}\nabla_{\mathbf{a}} Q^*(\mathbf{s},\mathbf{a}).
\end{equation}
In MaxEnt inverse RL (MaxEnt IRL), the reward function is unknown, while a set of expert demonstrations $\mathcal{D}_E=\{\tau_i\}_{i=1}^N$ is observed, where each trajectory $\tau=(\mathbf{s}_0,\mathbf{a}_0,\mathbf{s}_1,\mathbf{a}_1,\dots)$ is generated by an expert policy $\pi_E$. The goal is to recover a reward function under which $\pi_E$ is optimal under the above MaxEnt RL framework, which can be interpreted by the following optimization problem: 
\begin{equation}
\label{eq:maxent_irl_minimax}
\min_r \max_\pi \mathbb{E}_{(\mathbf{s},\mathbf{a})\sim\rho^\pi}[r(\mathbf{s},\mathbf{a})] + \varepsilon \, \mathcal{H}_\mu(\pi)
- 
\mathbb{E}_{(\mathbf{s},\mathbf{a})\sim\rho^{\pi_E}}[r(\mathbf{s},\mathbf{a})].
\end{equation}
The min-max structure of Eq.~\eqref{eq:maxent_irl_minimax} reveals an inherent {\em bi-level} optimization procedure of many classic and modern IRL methods \citep{ho2016generative,finn2016connection,fu2018learning,zeng2023demonstrations}: the upper level updates the reward by contrasting expert behavior with the learner behavior induced by the current reward, while the lower level updates the learner's imitation policy. Classic IRL methods typically implement this bi-level procedure as a computationally expensive {\em nested reward-policy loop}  because the lower level is a full RL procedure for computing an optimal policy. Some recent methods \citep{zeng2022maximum,zeng2023demonstrations} improve time efficiency by introducing a relatively cheap {\em single reward-policy loop}, where the lower level conducts one or several steps of policy improvement. 
Although this bi-level optimization structure is well theoretically grounded \citep{ho2016generative,zeng2022maximum} and has been shown to be effective in applications \citep{wu2020efficient,fulanguage}, its induced reward-policy loops introduce inevitable computational burden and instability to the optimization process.

\subsection{Diffusion Models and Diffusion Policy}

Diffusion models \citep{ho2020denoising,songscore} generate samples by modeling a {\em forward process} that gradually adds %Gaussian 
noise to data and a {\em reverse process} that iteratively removes this noise to recover data samples. {\em Diffusion policy} methods %~\footnote{Here, we introduce diffusion policy under the framework of Denoising Diffusion Probabilistic Models (DDPM)~\citep{ho2020denoising} because as of the time of this paper, almost all diffusion policy methods are built upon it. }~\citep{reuss2023goal,chi2025diffusion} 
use this idea for imitating expert behavior, where a policy $\pi(\cdot | \mathbf{s}_t)$ is modeled as an 
action generation process: %More formally, given the current state $\mathbf{s}_t$, 
given an MDP time step $t$, let $\{\mathbf{a}_t^k\}_{k=0}^K$ denote the latent action sequence indexed by the diffusion step $k$, where $\mathbf{a}_t^0$ corresponds to the initial action (samples from expert demonstrations) and larger $k$ corresponds to noisier ones.  
The forward process uses a variance schedule $\{\beta_k\}_{k=1}^K$ with $\alpha_k \coloneqq 1-\beta_k$, $\bar{\alpha}_k \coloneqq \prod_{j=1}^{k}\alpha_j$, and Gaussian noise  $\boldsymbol{\epsilon}\sim\mathcal{N}(0,I)$; the noisy action at step $k$ can be written in closed form as
\begin{equation}
\label{eq:diff_add_noise}
\mathbf{a}_t^k=\sqrt{\bar{\alpha}_k}\,\mathbf{a}_t^0+\sqrt{1-\bar{\alpha}_k}\,\boldsymbol{\epsilon}.
\end{equation} 
The reverse process trains a noise predictor $\boldsymbol{\epsilon}_{\bm \phi}$ by minimizing the following MSE loss: 
\begin{equation}
\label{eq:diff_policy_loss}
\mathcal{L}_{\mathrm{diff}}({\bm \phi})
\coloneqq
\mathbb{E}_{\substack{(\mathbf{s}_t,\mathbf{a}_t^0)\sim \rho^{\pi_E},\\
k\sim \mathrm{Unif}\{1,\dots,K\},\,\boldsymbol{\epsilon}\sim \mathcal{N}(0,I)}}
\left\|
\boldsymbol{\epsilon}
-
\boldsymbol{\epsilon}_{\bm \phi} \left( \mathbf{a}_t^k,
\mathbf{s}_t,
k
\right)
\right\|_2^2 .
\end{equation}
Training the noise predictor $\boldsymbol{\epsilon}_\phi$ can be interpreted as learning $\nabla_{\mathbf{a}}\log \pi_E(\mathbf{a} | \mathbf{s}_t)$, the so-called {\em score} of a diffusion policy, which is a vector field indicating the direction to adjust a noisy action to increase its likelihood under the expert action distribution. 
In the idealized limit of exact denoising, it can be shown that $\nabla_{\mathbf{a}}\log \pi_E(\mathbf{a} | \mathbf{s}_t) |_{\mathbf{a} = \mathbf{a}_t^k} = -\frac{1}{\sqrt{1-\bar{\alpha}_k}}
\boldsymbol{\epsilon}_{\bm \phi}(\mathbf{a}_t^k,\mathbf{s}_t,k) \eqqcolon g_{\bm \phi}(\mathbf{a}_t^k,\mathbf{s}_t,k)$.  
Combining this with Eq.~\eqref{eq:q_derivative}, under the condition that the expert policy is soft-optimal, i.e., $\pi_E=\pi^*$, %and the denoising is exact, 
we have
\begin{equation}
\label{eq:diffusion_q_connection}
%g_\phi(\mathbf{a}_t^k,\mathbf{s}_t,k)= {\textstyle \frac{1}{\varepsilon}}\nabla_{\mathbf{a}} Q_r^*(\mathbf{s}_t,\mathbf{a}) |_{\mathbf{a} = \mathbf{a}_t^k}.
g_\phi(\mathbf{a}_t^k,\mathbf{s}_t,k)
=
\nabla_{\mathbf{a}} \log \pi_E(\mathbf{a} | \mathbf{s}_t)
=
\nabla_{\mathbf{a}} \log \mu(\mathbf{a}|\mathbf{s}_t)|_{\mathbf{a}=\mathbf{a}_t^k}
+
{\textstyle \frac{1}{\varepsilon}}\nabla_{\mathbf{a}} Q_r^*(\mathbf{s}_t,\mathbf{a})|_{\mathbf{a}=\mathbf{a}_t^k}.
\end{equation}
Thus, a pretrained diffusion policy can provide the action-gradient supervision signal to recover the optimal soft Q function. We will utilize this property to design our method. To stay self-contained, detailed derivations of the identity relationship in Eq.~(\ref{eq:diffusion_q_connection}) are given in Appendix~\ref{app:diffusion_score_q}.

\section{Methods}\label{sec:methods}
In this section, we present \textbf{Loop-Free Inverse Reinforcement Learning (LFIRL)}, a fully offline IRL algorithm that avoids reward-policy loops in conventional IRL approaches. Instead, it proceeds in three sequential stages (see Fig.~\ref{fig:irl_comparison}): \textbf{(I)} first recover a soft Q function from diffusion-policy using Sobolev training  and recover the value function $V$ from the learned $Q$ in a way inspired by Gumbel regression; \textbf{(II)} then infer the state-dependent offset and calibrate the $Q$ and $V$; \textbf{(III)} finally recover the reward from the learned $Q$ and $V$ by enforcing Bellman consistency. 

\begin{figure}[t]
    \centering
    \includegraphics[width=0.8\linewidth]{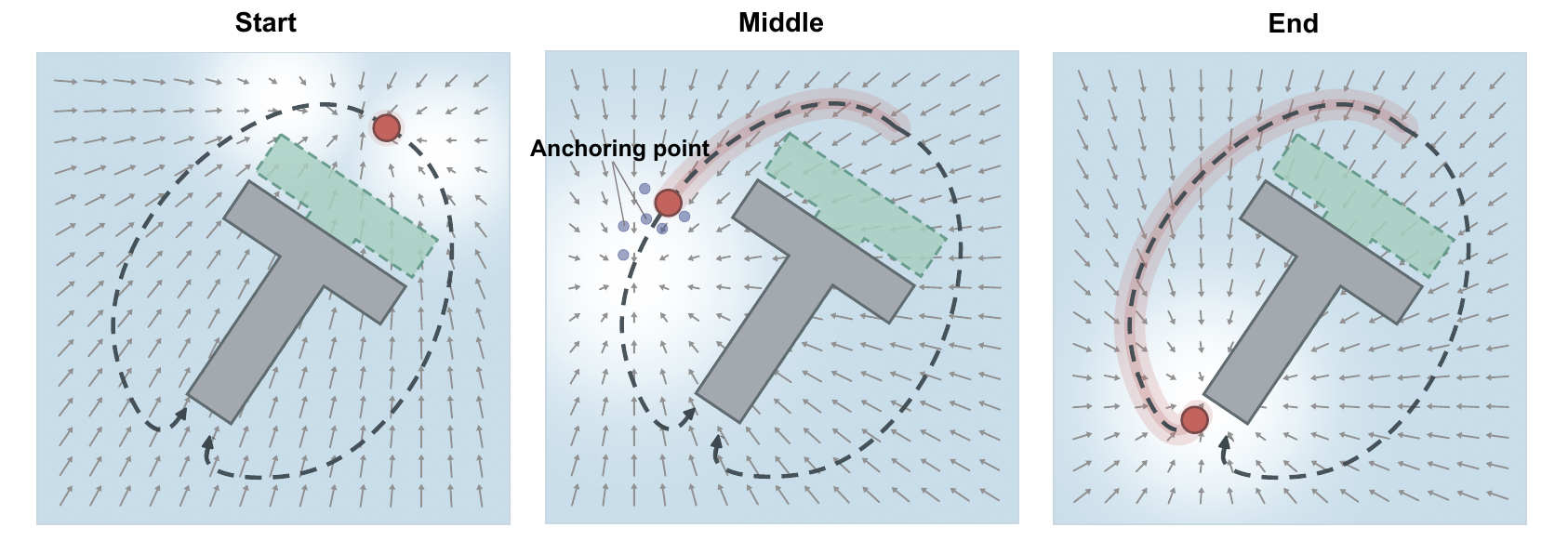}
    %\vspace{-1em}
    \caption{\small  Illustration of matching policy's score on Push-T task, where the goal is to push the T-shaped block to the {\mygreen target pose}. The {\myred red circle} marks the agent's current position (initially in the Start panel), and the dark dashed curves indicate two optimal routes to reach the correct position for pushing the block. The arrows depict the action-gradient field supervised by the  diffusion policy's score; the {\myblue blue dots} are perturbed actions whose Q values are constrained below that of the expert action by a margin, thereby providing anchoring information.} 
    %diffusion policy's score supervise the action-gradient of $\hat Q$ (gray arrows), while perturbed actions (pink anchoring points) help shape the local values of expert actions.} 
    \label{fig:gradient}
    %\vspace{-2em}
\end{figure}

\subsection{Stage I: Recover the Soft Q Function and Value function}\label{sec:qv} 
In this stage, we recover the optimal soft Q function through {\em Sobolev Training} \citep{czarnecki2017sobolev, son2021sobolev}, which trains neural networks by jointly matching function values and their derivatives. 
In our formulation, the score function of the diffusion policy provides direct supervision on action-gradients. 
Specifically, we train a Q-network $\hat{Q}_{\bm{\omega}}(\mathbf{s}, \mathbf{a})$ such that its action-gradient matches the diffusion policy's score after correcting for the reference policy.
Concretely, we align 
$\nabla_{\mathbf{a}} \hat{Q}_{\bm{\omega}}(\mathbf{s}, \mathbf{a})|_{\mathbf{a} = \mathbf{a}^k}$ 
with 
$g_{\bm{\phi}}(\mathbf{a}^k, \mathbf{s}, k) - \nabla_{\mathbf{a}}\log \mu(\mathbf{a}|\mathbf{s})|_{\mathbf{a}=\mathbf{a}^k}$, 
where $\mathbf{a}^k$ is obtained by adding noise to an expert action $\mathbf{a}^0$ via Eq.~\eqref{eq:diff_add_noise}. 

However, the derivation of value supervision is not straightforward, as the optimal Q values themselves are indeed what we seek for. To address this, we introduce an implicit form of value supervision by enforcing a small margin between the Q values of expert actions and those of their {\em locally} perturbed counterparts. We do not impose this margin globally, in order to preserve action diversity, as multiple actions may be optimal under the same state. Intuitively, this margin-based constraint effectively regularizes the Q function, using expert actions as ``anchor points'' to shape a more precise function surface. These settings, illustrated with Push-T task in Fig.~\ref{fig:gradient}, lead to the following loss function for training $\hat{Q}_{\bm{\omega}}(\mathbf{s}, \mathbf{a})$:

\begin{equation}
\label{eq:q_stage1_loss}
\begin{aligned}
\mathcal{L}_{\hat Q}({\bm\omega})
:=
&\;
\mathbb{E}_{\substack{\mathbf{s} \in \mathcal{S},\mathbf{a} \in \mathcal{A},\\k\sim \mathrm{Unif}\{1,\dots,K\}}}
\underbrace{\left[
\left\|
g_\phi(\mathbf{a}^k,\mathbf{s},k)
-
\nabla_\mathbf{a}\log \mu(\mathbf{a}|\mathbf{s}_t)|_{\mathbf{a}=\mathbf{a}_t^k}
-
{\textstyle \frac{1}{\varepsilon}} \nabla_{\mathbf{a}} \hat Q_{\bm\omega}(\mathbf{s},\mathbf{a})|_{\mathbf{a}=\mathbf{a}^k}
\right\|_2^2
\right]}_{\text{Gradient Matching}} \\
&\;+\lambda\,
\mathbb{E}_{\substack{(\mathbf{s},\mathbf{a})\sim \mathcal{D}_E,\\
\tilde{\mathbf{a}}\sim \eta(\cdot \mid \mathbf{a})}}
\underbrace{\left[
\operatorname{ReLU}\left(
\xi
-
\big(
\hat Q_{\bm\omega}(\mathbf{s},\mathbf{a})
-
\hat Q_{\bm\omega}(\mathbf{s},\tilde{\mathbf{a}})
\big)
\right)
\right]}_{\text{Value Anchoring}},
\end{aligned}
\end{equation}
where $\xi>0$ is an intended margin, $\eta (\cdot|\mathbf{a})$ denotes the  neighborhood of action $\mathbf{a}$, 
$\mathrm{ReLU}(x) \coloneqq \max\{0, x\}$ denotes the Rectified Linear Unit function,
and $\lambda\ge 0$ is a weighting coefficient. In practice, we can instantiate the reference policy $\mu(\mathbf{a}|\mathbf{s})$ 
as a simple distribution (e.g., uniform over bounded actions or a Gaussian prior), 
whose score $\nabla_{\mathbf{a}} \log \mu(\mathbf{a}|\mathbf{s})$ is tractable.

After obtaining $\hat Q_{\bm\omega}$, we estimate the corresponding soft value function. However, direct computation of $\hat{V}(\mathbf{s})=
\varepsilon \log \int_\mathcal{A}
\exp(\hat Q_{\bm\omega}(\mathbf{s},\mathbf{a})/\varepsilon)
\, d \mu(\mathbf{a}|\mathbf{s}) $
is generally intractable in continuous action spaces. 
To address this, we reformulate the problem via an equivalent moment condition, drawing inspirations from  {\em Gumbel regression} in RL \citep{gargextreme}, which models action selection by perturbing  Q values with Gumbel distribution noise \citep{hazan2012partition} and leverages {\em Gumbel-Max Trick} \citep{hazan2012partition} to infer the value function as the expected stochastic maximum, yielding a LogSumExp form.
Specifically, the soft value $\hat V(\mathbf{s})$ is the unique solution to 
$
\mathbb{E}_{\mathbf{a}\sim\mu(\cdot|\mathbf{s})}
[
\exp(\frac{\hat Q_{\bm\omega}(\mathbf{s},\mathbf{a}) - V(\mathbf{s})}{\varepsilon})
]
= 1.
$
This characterization allows us to estimate $\hat V(\mathbf{s})$ without explicitly computing the LogSumExp. 
We then construct a convex surrogate whose first-order condition recovers the above moment constraint:
\begin{equation}
\label{eq:v_pointwise_loss_refined}
\ell_V(v;\mathbf{s})
\coloneqq
\mathbb{E}_{\mathbf{a}\sim\mu(\cdot|\mathbf{s})}
\left[
\exp \left(\frac{\hat Q_{\bm\omega}(\mathbf{s},\mathbf{a})-v}{\varepsilon}\right)
-
\frac{\hat Q_{\bm\omega}(\mathbf{s},\mathbf{a})-v}{\varepsilon}
- 1
\right].
\end{equation}
This objective is strictly convex in $v$ and admits a unique minimizer, which coincides with the soft value $\hat V(\mathbf{s})$. The derivation is provided in Appendix~\ref{app:q_recovery_derivations}. 
In practice, we freeze $\hat Q_{\bm\omega}$ and train a value network $\hat V_{\bm\chi}$ on a set of trajectories $\mathcal D_S$ sampled from $\mu$ through the following empirical objective
\begin{equation}
\label{eq:v_loss_main_refined}
\mathcal{L}_{\hat V}({\bm\chi})
\coloneqq
\mathbb{E}_{(\mathbf{s},\mathbf{a})\sim D_S}
\left[
\exp(z)-z-1
\right],
\quad
z =
\frac{\hat Q_{\bm\omega}(\mathbf{s},\mathbf{a})-\hat V_{\bm\chi}(\mathbf{s})}{\varepsilon}.
\end{equation}
Notably, although this objective coincides with the Gumbel regression loss \citep{gargextreme}, our derivation does not rely on 
a stochastic model for Q values, as $\hat Q_{\bm \omega}$ is fixed during training $\hat V_{\bm \chi}$; instead, it follows directly from the deterministic characterization of 
the LogSumExp under the reference policy $\mu$.

\subsection{Stage II: Infer the State-Dependent Offset and Calibrate Soft Values}
\label{sec:offset}

The soft values recovered in Stage I are {\em uncalibrated}, as the action-gradient matching determines only how soft values change wrt actions, but does not reflect their state-dependent factors. This fact can be formally explained by the following proposition, whose justification is given in Appendix~\ref{app:offset_v_derivations}.

\begin{proposition}
\label{prop:q_up_to_b}
Assume $\pi_E=\pi^*$ and that the learned diffusion policy's score is exact in the sense of Eq.~\eqref{eq:diffusion_q_connection}. For each state $\mathbf{s}$, if the action-gradient matching term in Eq.~\eqref{eq:q_stage1_loss} is minimized exactly, then there exists a state-dependent function $b^*:\mathcal{S}\to\mathbb{R}$ such that $Q^*(\mathbf{s},\mathbf{a}) = \hat Q_{\bm\omega}(\mathbf{s},\mathbf{a}) + b^*(\mathbf{s})$.
\end{proposition}
This reflects the ambiguity caused by \emph{shift invariance}: adding any {\em state-dependent offset} $b(\mathbf{s})$ leaves action-gradient unchanged, i.e., $\nabla_{\mathbf{a}} \hat Q_{\bm\omega}(\mathbf{s},\mathbf{a})=\nabla_{\mathbf{a}} \big(\hat Q_{\bm\omega}(\mathbf{s},\mathbf{a}) + b(\mathbf{s})\big)$, 
making $\hat Q_{\bm\omega}$ identifiable only up to $b(\mathbf{s})$. 
We therefore introduce an offset network $b_{\bm\psi}:\mathcal{S}\to\mathbb{R}$ and define the calibrated soft Q function as $Q_{{\bm\omega},{\bm\psi}}(\mathbf{s},\mathbf{a})
\coloneqq
\hat Q_{\bm\omega}(\mathbf{s},\mathbf{a})
+
b_{\bm\psi}(\mathbf{s})$ and write the calibrated soft value function from Stage I as $\hat V_{\bm\chi}(\mathbf{s})+b_{\bm\psi}(\mathbf{s})$. 
A key challenge is that the offset $b_{\bm\psi}$ is not identifiable from action-gradient information. While penalizing $b_{\bm\psi}^2$ can fix the shift ambiguity by selecting the minimum-norm solution among all equivalent offsets, it ignores the transition structure and may yield solutions that violate Bellman consistency. Therefore, we regularize the {\em Bellman-implied reward}, 
\begin{equation}
\delta_{{\bm\omega},{\bm\chi},{\bm\psi}}(\mathbf{s}_t,\mathbf{a}_t,\mathbf{s}_{t+1})
\coloneqq
\hat Q_{\bm\omega}(\mathbf{s}_t,\mathbf{a}_t)
+
b_{\bm\psi}(\mathbf{s}_t)
-
\gamma
\big(
\hat V_{\bm\chi}(\mathbf{s}_{t+1})
+
b_{\bm\psi}(\mathbf{s}_{t+1})
\big),
\end{equation}
which couples $b_{\bm\psi}$ across consecutive states and enforces consistency with the underlying dynamics. Intuitively, minimizing the squared norm of $\delta_{{\bm\omega},{\bm\chi},{\bm\psi}}$ penalizes inconsistent offset value shifts across transitions, encouraging the offset to vary coherently along trajectories. %This transforms offset inference into a structured inverse problem, where TD consistency provides the governing constraint, and a mild $\ell_2$ penalty on $b_{\bm\psi}$ fixes the remaining degree of freedom. 

In addition, absorbing states provide direct supervision: for any $\mathbf{s}_\bot \in \mathcal{S}_\bot$, we have $Q^*(\mathbf{s}_\bot,\mathbf{a})=0$ for all $\mathbf{a}$.~\footnote{This zero value is caused by the absorbing state itself, not by a particular action, so it is used to calibrate the state-dependent offset rather than to provide the action-level value supervision used in Stage I.} Thus, the resulting objective is
%Therefore, we use $\delta_{{\bm\omega},{\bm\chi},{\bm\psi}}$ to constrain the difference between the calibrated current $Q$-value and the discounted offset-corrected next-state value, while the quadratic penalty prevents the offset from becoming arbitrarily large. In addition, absorbing states provide direct supervision for $Q$: if $\mathbf{s}_\bot$ is an absorbing state, then $Q^*(\mathbf{s}_\bot,\mathbf{a})=0$ for any valid action $\mathbf{a}$.~\footnote{This zero value is determined by the absorbing state itself, rather than by the choice of action. Therefore, it cannot serve as the action-level value supervision in Stage I.} 
%The resulting loss function is 
\begin{equation}
\label{eq:b_loss_main}
\begin{aligned}
\mathcal{L}_{b}({\bm\psi})
\coloneqq
\mathbb{E}_{\substack{(\mathbf{s}_t,\mathbf{a}_t,\mathbf{s}_{t+1})\\\sim\mathcal D_E \cup \mathcal{D}_S}}
[
\underbrace{
\delta^2_{{\bm\omega},{\bm\chi},{\bm\psi}}
+
\lambda_b\, b_{\bm\psi}^2(\mathbf{s}_t)
}_{\text{Offset regularization}}
]
+
&\;
\mathbb{E}_{\mathbf{s}_\bot \in \mathcal{S}_\bot,\mathbf{a}}
[
\underbrace{
\big(
\hat Q_{\bm\omega}(\mathbf{s}_\bot,\mathbf{a})
+
b_{\bm\psi}(\mathbf{s}_\bot)
\big)^2
}_{\text{Absorbing states supervision}}],
\end{aligned}
\end{equation}
where $\lambda_b\ge 0$ is a regularization coefficient, $\hat Q_{\bm\omega}$ and $\hat V_{\bm\chi}$ are fixed, and only $b_{\bm\psi}$ is updated.

After calibrating the soft Q function, we re-estimate the value function using the calibrated $Q_{{\bm\omega},{\bm\psi}}$. Specifically, we freeze $Q_{{\bm\omega},{\bm\psi}}$ and optimize the Gumbel regression objective again:
\begin{equation}
\label{eq:v_loss_calibrated}
\mathcal{L}_{V}({\bm\chi}')
\coloneqq
\mathbb{E}_{\mathbf{s},\mathbf{a}\sim\mu(\cdot|\mathbf{s})}
\left[
\exp(\bar z)-\bar z-1
\right],
\qquad
\bar z
\coloneqq
\frac{
Q_{{\bm\omega},{\bm\psi}}(\mathbf{s}_t,\mathbf{a}_t)
-
V_{\bm\chi'}(\mathbf{s}_t)
}{\varepsilon}.
\end{equation}
This second value-fitting step is needed because the soft value function must be consistent with the offset-calibrated soft Q function. 
%The resulting pair $Q_{{\bm\omega},{\bm\psi}}$ and $V_{\bm\chi'}$ is then used for reward recovery in Stage III.

\subsection{Stage III: Recover the Reward Function}\label{sec:r}

Given the calibrated soft Q function $Q_{{\bm\omega},{\bm\psi}}$ and value function $V_{\bm\chi'}$, reward recovery reduces to enforcing Bellman consistency. For a transition $(\mathbf{s}_t,\mathbf{a}_t,\mathbf{s}_{t+1})$, we write the Bellman-implied reward target
$
r=Q_{{\bm\omega},{\bm\psi}}(\mathbf{s}_t,\mathbf{a}_t)-\gamma V_{\bm\chi'}(\mathbf{s}_{t+1}).
$ 
We parameterize the reward by a network $r_{\bm\theta}$ and learn it by minimizing Bellman residual errors. To stabilize training, we apply clipping and normalization:
\begin{equation}
\tilde r
=
\frac{\operatorname{clip}(r,-c_r,c_r)-m_r}{\sigma_r+\zeta_r},
\qquad
\tilde r_{\bm\theta}
=
\frac{r_{\bm\theta}(\mathbf{s}_t,\mathbf{a}_t)-m_r}{\sigma_r+\zeta_r},
\end{equation}
where $c_r>0$ is the clipping threshold, $(m_r,\sigma_r)$ are the running mean and standard deviation of the clipped reward targets, and $\zeta_r>0$ is a small numerical constant.
The reward network is trained via
\begin{equation}
\label{eq:r_loss_main}
\mathcal{L}_{r}({\bm\theta})
=
\mathbb{E}_{(\mathbf{s}_t,\mathbf{a}_t,\mathbf{s}_{t+1})\sim\mathcal D_E \cup \mathcal D_S}
\left[
( \tilde r_{\bm\theta} - \tilde r )^2
\right].
\end{equation}

\begin{algorithm}[t]
    \caption{Loop-Free Inverse Reinforcement Learning (LFIRL)}
    \label{alg:loop_free_irl}
    %\small
    \begin{algorithmic}[1]
       \STATE {\bfseries Input:} Expert demonstrations $\mathcal{D}_E$; initialized networks $\hat Q_{\bm\omega}$, $b_{\bm\psi}$, $V_{\bm\chi}$, and $r_{\bm\theta}$; reference policy $\mu$.%discount factor $\gamma$.
       %\STATE Construct an expert transition buffer $\mathcal{B}_E$ from $\mathcal{D}_E$ and a .
       \STATE {\mygreen Train a (or load a pre-trained) diffusion policy on $\mathcal{D}_E$. \COMMENT{\textbf{Stage I}}}
       \STATE {\mygreen Recover $\hat Q_{\bm\omega}(\mathbf{s}, \mathbf{a})$ on $\mathcal{D}_E$ via matching diffusion policy's score. }\hfill$\rhd$ Eq.~\eqref{eq:q_stage1_loss} 
       \STATE {\mygreen Recover $V_{\bm\chi}(\mathbf{s})$ on $\mathcal{D}_S$ using Gumbel-regression-style value fitting.} \hfill $\rhd$ Eq.~\eqref{eq:v_loss_main_refined}
       \STATE {\myorange Recover state-dependent offset $b_{\bm\psi}(\mathbf{s})$ on $\mathcal{D}_E$ and $\mathcal{D}_S$. \COMMENT{\textbf{Stage II}}} \hfill $\rhd$ Eq.~\eqref{eq:b_loss_main}
       \STATE {\myorange Calibrate $\hat V_{\bm\chi'}(\mathbf{s})$ by the calibrated Q values $\hat Q_{\bm\omega} + b_{\bm\psi}(\mathbf{s})$.} \hfill $\rhd$ Eq.~\eqref{eq:v_loss_calibrated}
       \STATE {\myblue Recover $r_{\bm\theta}$ on $\mathcal{D}_E$ and $\mathcal{D}_S$ by minimizing Bellman residual errors. \COMMENT{\textbf{Stage III}}} \hfill $\rhd$ Eq.~\eqref{eq:r_loss_main}
       \STATE {\bfseries Output:} recovered reward function $r_{\bm\theta}$.
    \end{algorithmic}
\end{algorithm}

\subsection{Algorithm Summary}
The complete procedure of the proposed Loop-Free IRL framework is summarized in Algorithm~\ref{alg:loop_free_irl}. 
In contrast to traditional IRL pipelines, our method eliminates the reward-policy loop entirely: each stage is executed only once in a sequential manner, and every component is trained solely on the frozen outputs produced by preceding stages. Importantly, the algorithm is {\em fully offline}, as $\mathcal{D}_S$ is an offline dataset collected from the reference policy $\mu$. Consequently, training requires neither environment interactions nor rollouts from a learned dynamics model during training. Overall, this loop-free offline design substantially improves training efficiency, as will be demonstrated experimentally.

\section{Experiments}\label{sec:experiment}

We evaluate LFIRL from two perspectives: {\bf (1)} reward recovery quality and {\bf (2)} training time efficiency.

\subsection{Experimental Setup}

We consider four different kinds of tasks (see Fig.~\ref{fig:tasks}): PointMaze with the UMaze, Medium, and Large settings; Franka Kitchen consists of four subtasks; Adroit Hand Pen; and the Push-T manipulation \citep{chi2025diffusion}. We compare against four classes of baselines: (1) {\em Non-Adversarial IRL} baselines: IQ-Learn \citep{garg2021iq} and ML-IRL \citep{zeng2022maximum}; (2) the {\em Adversarial IRL} baseline: AIRL \citep{fu2018learning}; (3) the {\em Offline IRL} baseline: Offline ML-IRL \citep{zeng2023demonstrations}, CLARE \citep{yue2023clare}, and ValueDICE \citep{kostrikov2019imitation}; and (4) {\em Diffusion-based GAIL} baselines: DRAIL \citep{lai2024diffusion} and DIFO \citep{huang2024diffusion}. All reported results are averaged over five random runs. All methods are trained with the same environment-step budget, and within each environment we use the same number of expert demonstrations for all methods. Full implementation details, including hyperparameters and network architectures, are in Appendix~\ref{app:setup}.

To evaluate reward recovery on PointMaze and Franka Kitchen, we
train an SAC \citep{haarnoja2018soft} policy on the recovered reward, and then evaluate this policy in the original environment. We report success rate on PointMaze and the mean number of completed subtasks on Franka Kitchen. Since our method is designed to recover a reward function rather than directly output a policy, this setting provides a task-level assessment of whether the recovered reward induces the intended behavior. For the higher-dimensional Adroit Hand Pen and Push-T environments, we instead evaluate reward quality through discrimination: we construct balanced sets of high-quality and low-quality trajectories, score them with each learned reward model, and report the resulting classification accuracy. We additionally compare the recovered and ground-truth rewards on the same state-action samples using Pearson's and Spearman's correlation coefficients; the results are presented in Sec.~\ref{sec:direct_reward_recovery}. To evaluate efficiency, we record the training time required by each method to consume the same number of trajectory throughput, with each method run independently and without parallelized competition for compute resources. In addition to the main comparison, we also report an offset-ablation study and a demonstration-reduction study; the corresponding results are provided in Sec.~\ref{sec:ablation} and Sec.~\ref{sec:fewer_demo}.

\subsection{Reward Recovery}

The main results are shown in Fig.~\ref{fig:reward_recovery}. Overall, LFIRL recovers highly effective rewards and exhibits strong cross-task robustness. Across various tasks, the recovered reward consistently supports competitive downstream control and achieves the best or near-best final performance in most settings. Even in cases where it does not attain the single best score, its performance remains highly competitive, indicating that the proposed sequential recovery pipeline yields stable and transferable reward signals.

\begin{figure}[t]
  \centering
  \includegraphics[width=0.155\linewidth]{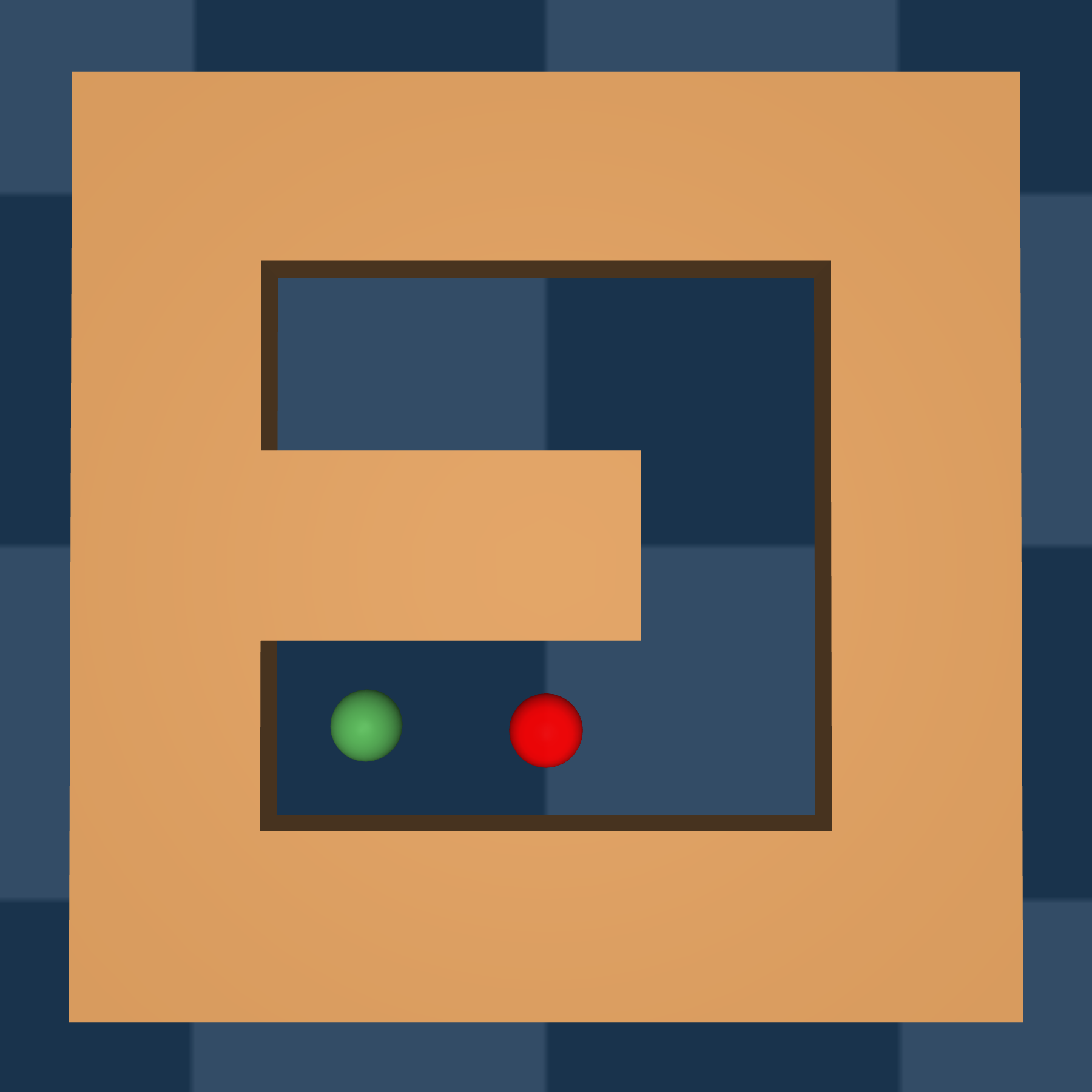}\hfill
  \includegraphics[width=0.155\linewidth]{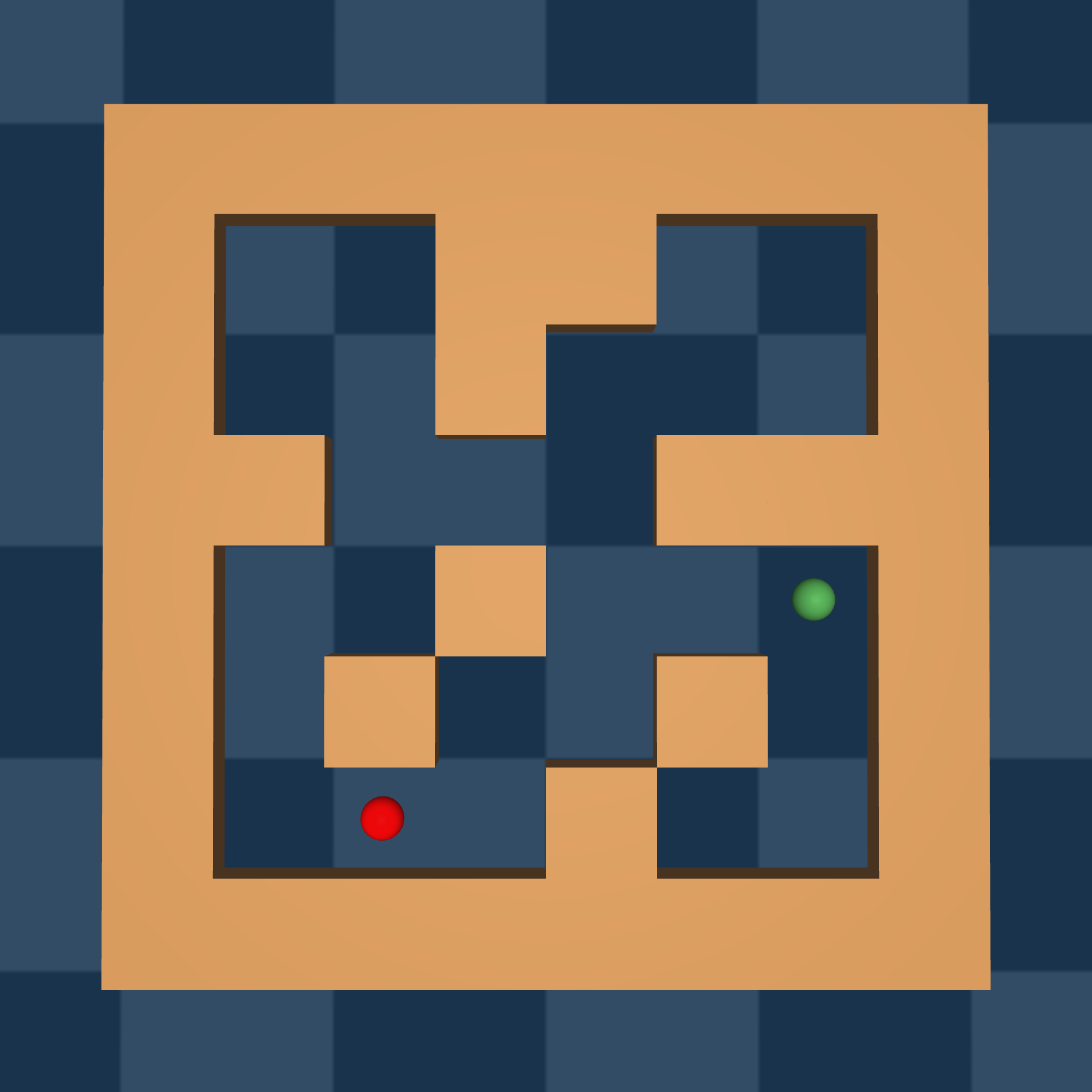}\hfill
  \includegraphics[width=0.155\linewidth]{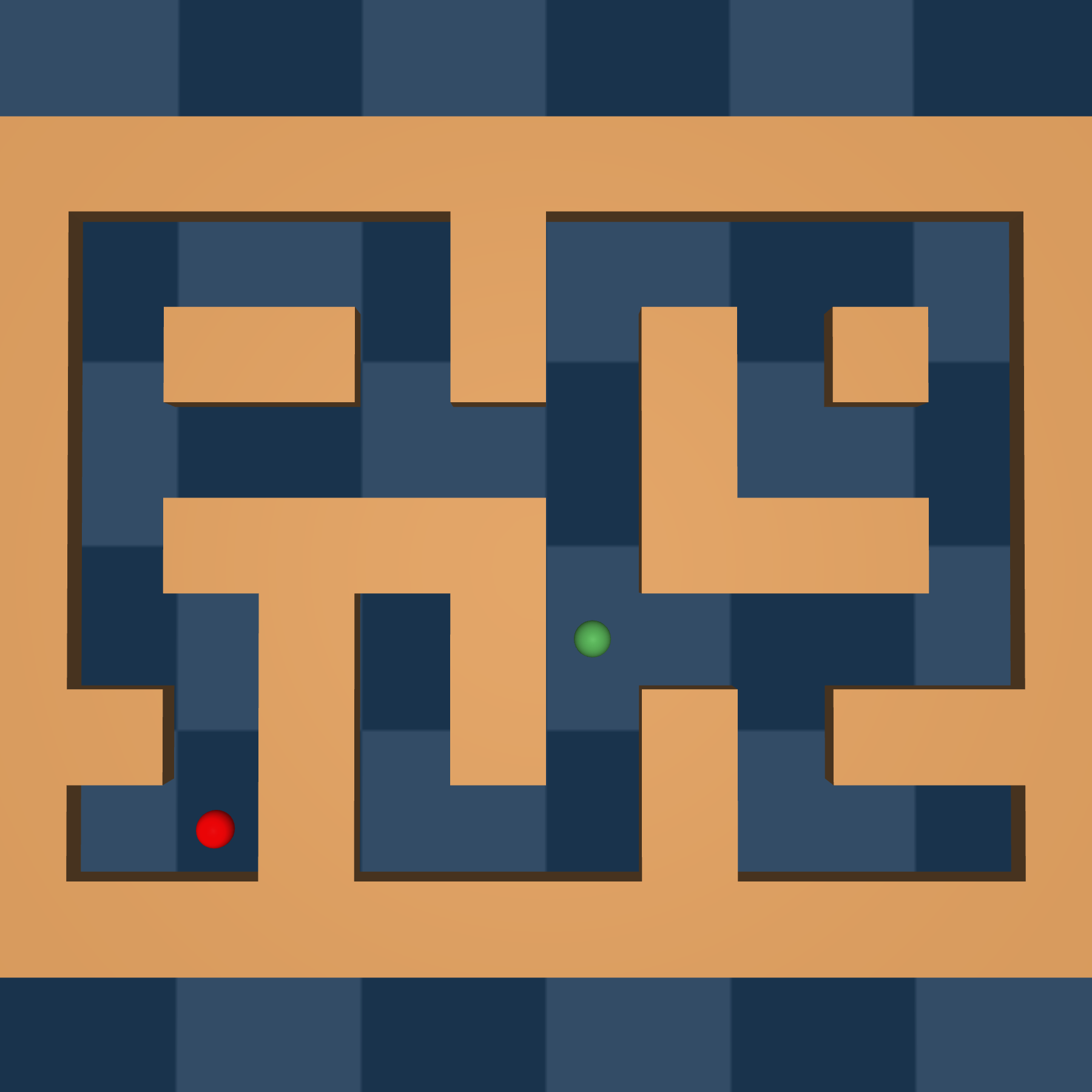}\hfill
  \includegraphics[width=0.155\linewidth]{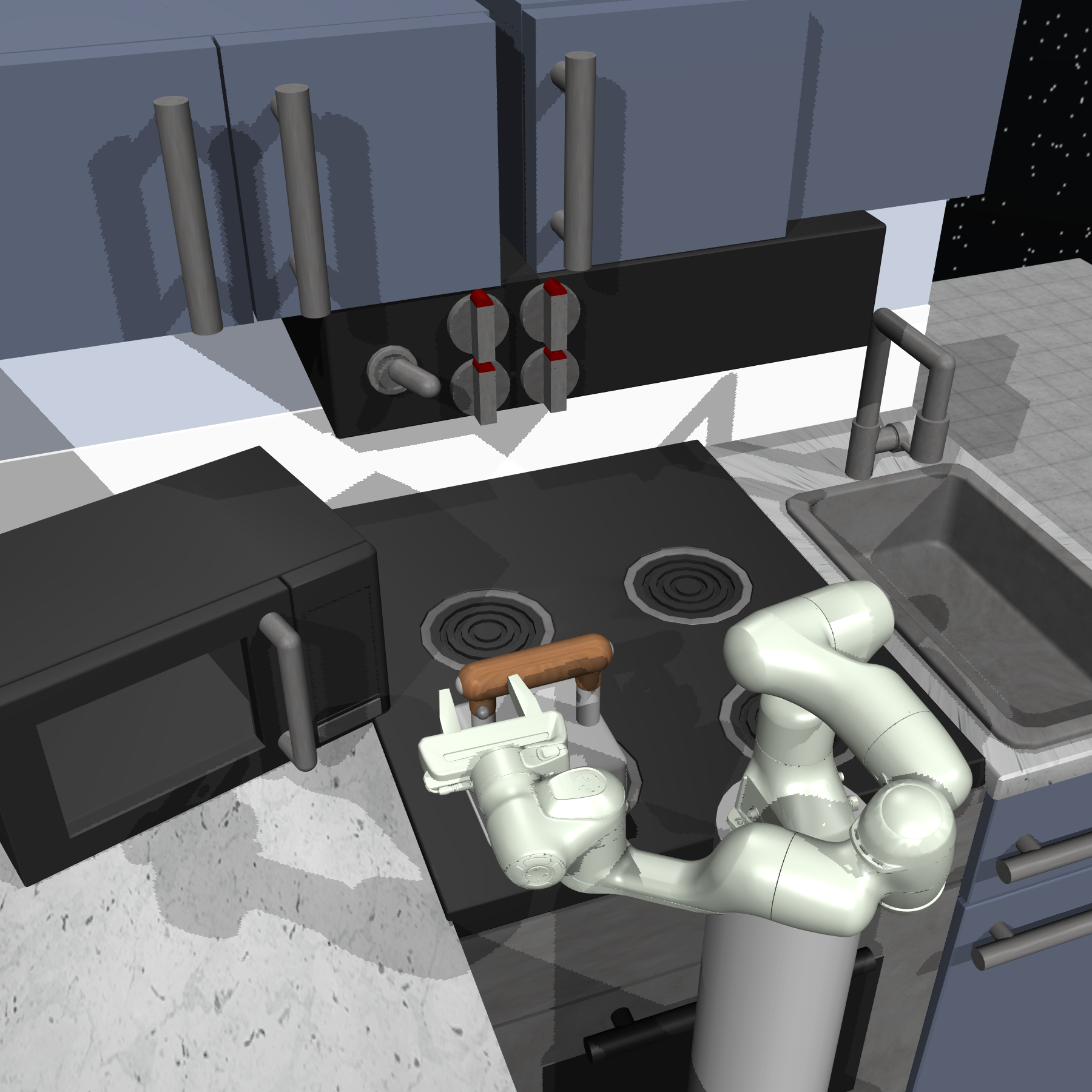}\hfill
  \includegraphics[width=0.155\linewidth]{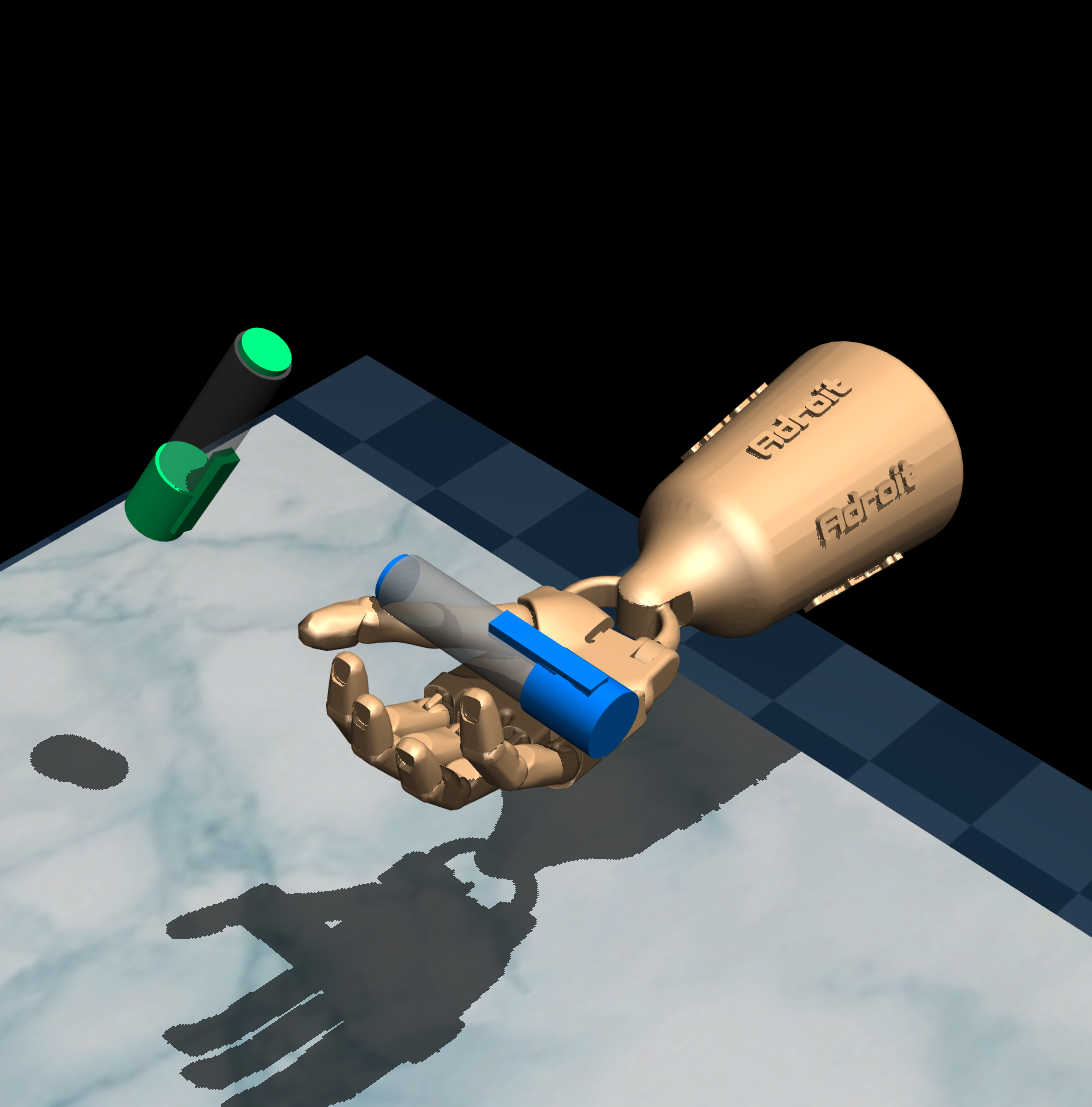}\hfill
  \includegraphics[width=0.155\linewidth]{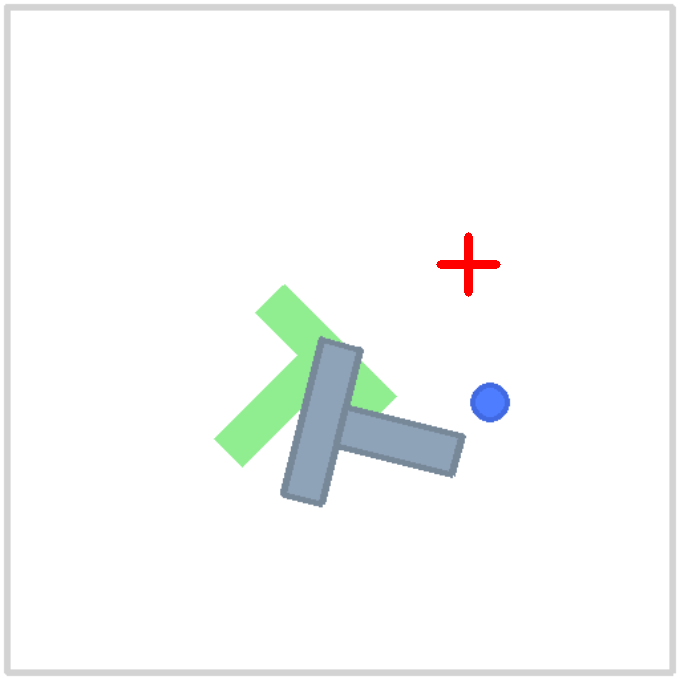}

  \vspace{2pt}

  \parbox[b]{0.155\linewidth}{\centering UMaze}\hfill
  \parbox[b]{0.155\linewidth}{\centering Maze(M)}\hfill
  \parbox[b]{0.155\linewidth}{\centering Maze(L)}\hfill
  \parbox[b]{0.155\linewidth}{\centering Kitchen}\hfill
  \parbox[b]{0.155\linewidth}{\centering Pen}\hfill
  \parbox[b]{0.155\linewidth}{\centering Push-T}

  \caption{\small Illustrations of the evaluation tasks used in our experiments.}
  \label{fig:tasks}
\end{figure}

\begin{figure}[tb]
    \centering
    \includegraphics[width=\columnwidth]{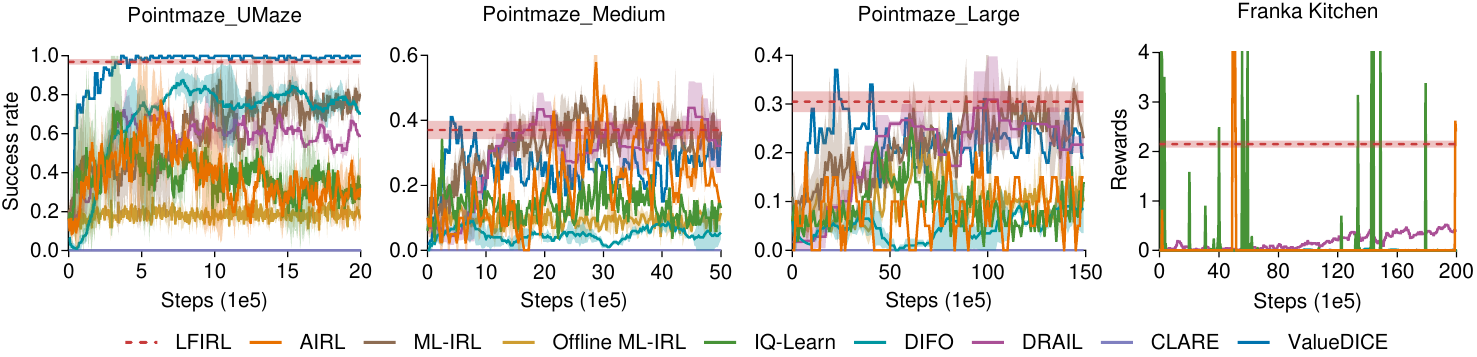}
    \caption{\small
    Reward-recovery results on PointMaze and Franka Kitchen. We report success rate on PointMaze and average completed subtasks on Franka Kitchen. Since LFIRL recovers the reward only at the final stage, its result is shown as a fixed final score rather than a policy-learning curve. For online IRL methods, ``steps'' denotes environment interaction steps. For offline IRL methods, ``steps'' denotes the number of transitions sampled from the offline dataset during training.
    }
    \label{fig:reward_recovery}
\end{figure}

For the higher-dimensional Push-T and Adroit Hand Pen, we instead evaluate whether the recovered reward can distinguish high-quality trajectories from low-quality ones. The results are reported in Tab.~\ref{tab:pusht_recovery}. These results further support a consistent overall picture: the reward recovered by LFIRL captures meaningful behavioral differences and generalizes well across qualitatively different evaluation protocols, including both downstream control and trajectory discrimination.

\begin{table}[tb]
  \centering
  \scriptsize
  \setlength{\tabcolsep}{2.5pt}
  \renewcommand{\arraystretch}{0.95}
  \caption{\small
    Trajectory-quality classification accuracy (\%) on Push-T and Adroit Hand Pen using reward networks recovered by different IRL algorithms. Results are reported as mean $\pm$ standard deviation over five runs.
  }
  \label{tab:pusht_recovery}
  \vspace{1em}
  \begin{tabular}{lccccccccc}
    \toprule
    \textbf{Task}
      & LFIRL
      & AIRL
      & IQ-Learn
      & DIFO
      & DRAIL
      & ML-IRL
      & Offline ML-IRL
      & CLARE
      & ValueDICE \\
    \midrule
    Push-T
      & \textbf{74.17$\pm$1.96}
      & 59.67$\pm$2.08
      & 39.67$\pm$0.58
      & \textbf{79.00$\pm$1.00}
      & 49.67$\pm$1.15
      & 71.33$\pm$2.33
      & 42.67$\pm$1.20
      & 71.00$\pm$3.67
      & \textbf{78.67$\pm$2.52} \\
    Pen
      & \textbf{97.17$\pm$2.75}
      & 90.00$\pm$2.45
      & 74.33$\pm$7.59
      & 91.67$\pm$2.05
      & \textbf{98.33$\pm$0.94}
      & 71.00$\pm$4.32
      & 66.00$\pm$12.43
      & 94.67$\pm$0.58
      & 95.67$\pm$1.15 \\
    \bottomrule
  \end{tabular}
\end{table}

\begin{table}[ht]
  \centering
  \scriptsize
  \setlength{\tabcolsep}{1.5pt}
  \renewcommand{\arraystretch}{1}
  \caption{\small
    Average running time (in hours). ``Load D.P.'' denotes LFIRL with a pre-trained diffusion policy, while ``Train D.P.'' denotes LFIRL including diffusion-policy training time. ``T.R.'' is short for ``time reduction'' and reports the relative time saved compared with the fastest baseline (marked with an underline).
  }
  \label{tab:running_time}
  \vspace{1em}
  \begin{tabular}{l|rrrrr|rrr|rr|rr}
    \toprule
   \multirow{2}{*}{\textbf{Task}}
      & \multicolumn{5}{c|}{\textbf{Online methods}}
      & \multicolumn{3}{c|}{\textbf{Offline methods}}
      & \multicolumn{4}{c}{\textbf{LFIRL}} \\
    \cmidrule(lr){2-6}
    \cmidrule(lr){7-9}
    \cmidrule(lr){10-13}
      & AIRL
      & IQ-Learn
      & DIFO
      & DRAIL
      & ML-IRL
      & ML-IRL
      & CLARE
      & ValueDICE
      & Load D.P.
      & T.R. (SpeedUp)
      & Train D.P.
      & T.R. (SpeedUp) \\
    \midrule
    UMaze
      & \underline{1.38} & 6.88 & 2.47 & 1.70 & 9.76
      & 3.09 & 5.86 & 4.66
      & 0.28 & 79.7\% (5x) & 0.43 & 68.8\% (3x) \\
    Maze(M)
      & 3.67 & 10.10 & 7.05 & \underline{2.65} & 13.80
      & 4.99 & 5.98 & 4.87
      & 0.77 & 70.9\% (3x) & 1.07 & 59.6\% (2x) \\
    Maze(L)
      & 6.88 & 10.88 & 7.73 & \underline{3.05} & 13.33
      & 5.52 & 6.67 & 5.17
      & 1.42 & 53.4\% (2x) & 1.80 & 41.0\% (2x) \\
    Kitchen
      & 9.55 & 13.72 & 7.40 & \underline{3.33} & 18.42
      & 6.83 & 7.80 & 5.91
      & 1.52 & 54.4\% (2x) & 1.87 & 43.8\% (2x) \\
    Push-T
      & \underline{1.57} & 8.55 & 4.05 & 1.95 & 9.50
      & 2.10 & 3.20 & 5.30
      & 0.35 & 77.7\% (4x) & 0.62 & 60.5\% (3x) \\
    Pen
      & 4.83 & 14.88 & 10.83 & 7.50 & 18.52
      & 6.53 & \underline{3.79} & 4.71
      & 1.18 & 68.9\% (3x) & 1.71 & 54.9\% (2x) \\
    \bottomrule
  \end{tabular}
\end{table}

\subsection{Average Running Time}

Tab.~\ref{tab:running_time} reports the average running time under the same step budget. Overall, LFIRL substantially reduces the time cost of training. Across tasks, it is consistently the fastest method, typically achieving about a 2--3x reduction in runtime compared with the fastest baseline in each setting, and a roughly 2--5x reduction when loading a pre-trained diffusion policy. These results confirm that removing the reward-policy optimization loop brings not only pipeline simplicity, but also a substantial practical gain in time efficiency.

\subsection{Correlation with Ground-Truth Rewards}
\label{sec:direct_reward_recovery}

The preceding evaluations assess whether a recovered reward supports the intended behavior or distinguishes trajectories of different quality. We further evaluate reward recovery directly by comparing each learned reward with the environment's ground-truth reward on the same set of state-action samples. We report Pearson's correlation coefficient (PCC), which measures linear association between the recovered and ground-truth reward values, and Spearman's rank correlation coefficient (SCC), which measures agreement between their rankings. A higher PCC indicates that the recovered reward more closely follows the ground-truth reward's linear variation across samples; a higher SCC indicates that it better preserves which samples receive higher or lower reward. These measures assess agreement in structure and ordering, respectively, without requiring the recovered reward to have exactly the same scale or offset as the ground-truth reward.

\begin{table*}[t]
  \centering
  \scriptsize
  \setlength{\tabcolsep}{5pt}
  \renewcommand{\arraystretch}{1.08}
  \caption{\small
    Direct comparison of recovered and ground-truth rewards on the same state-action samples. Each entry reports PCC / SCC; higher values indicate stronger linear association / rank agreement. Bold indicates the highest value for each metric in each task.
  }
  \label{tab:direct_reward_recovery}
  \vspace{0.5em}
  \begin{tabular}{lccccccc}
      \toprule
      \textbf{Method}
        & \textbf{UMaze}
        & \textbf{Medium}
        & \textbf{Large}
        & \textbf{Kitchen}
        & \textbf{Push-T}
        & \textbf{Pen}
        & \textbf{Average} \\
      & \textbf{PCC / SCC}
      & \textbf{PCC / SCC}
      & \textbf{PCC / SCC}
      & \textbf{PCC / SCC}
      & \textbf{PCC / SCC}
      & \textbf{PCC / SCC}
      & \textbf{PCC / SCC} \\
      \midrule
    LFIRL (Ours)
      & \textbf{0.83 / 0.91}
      & \textbf{0.76} / 0.82
      & \textbf{0.84 / 0.87}
      & \textbf{0.67 / 0.81}
      & 0.72 / \textbf{0.81}
      & \textbf{0.79 / 0.85}
      & \textbf{0.77 / 0.85} \\
    \midrule
    AIRL
      & 0.65 / 0.86
      & 0.63 / 0.74
      & 0.77 / 0.75
      & 0.44 / 0.63
      & 0.69 / 0.74
      & 0.75 / 0.71
      & 0.66 / 0.74 \\
    IQ-Learn
      & 0.62 / 0.74
      & 0.59 / 0.62
      & 0.65 / 0.70
      & 0.34 / 0.51
      & 0.48 / 0.65
      & 0.47 / 0.62
      & 0.53 / 0.64 \\
    DIFO
      & 0.69 / 0.79
      & 0.53 / 0.67
      & 0.68 / 0.81
      & 0.61 / 0.77
      & \textbf{0.83} / 0.79
      & 0.71 / 0.72
      & 0.68 / 0.76 \\
    DRAIL
      & 0.67 / 0.71
      & 0.74 / \textbf{0.84}
      & 0.69 / 0.78
      & 0.56 / 0.76
      & 0.51 / 0.56
      & 0.63 / 0.62
      & 0.63 / 0.71 \\
    ML-IRL
      & 0.78 / 0.75
      & 0.67 / 0.72
      & 0.64 / 0.71
      & 0.63 / 0.69
      & 0.66 / 0.73
      & 0.78 / 0.69
      & 0.69 / 0.72 \\
    Offline ML-IRL
      & 0.49 / 0.52
      & 0.47 / 0.58
      & 0.42 / 0.62
      & 0.45 / 0.54
      & 0.45 / 0.53
      & 0.53 / 0.43
      & 0.47 / 0.54 \\
    CLARE
      & 0.72 / 0.80
      & 0.68 / 0.74
      & 0.49 / 0.61
      & 0.24 / 0.46
      & 0.51 / 0.79
      & 0.58 / 0.55
      & 0.54 / 0.66 \\
    ValueDICE
      & 0.72 / 0.81
      & 0.61 / 0.79
      & 0.74 / 0.71
      & 0.53 / 0.56
      & 0.65 / 0.76
      & 0.57 / 0.63
      & 0.64 / 0.71 \\
    \bottomrule
  \end{tabular}
\end{table*}

As shown in Tab.~\ref{tab:direct_reward_recovery}, LFIRL obtains the highest PCC on five of the six tasks and the highest SCC on five of the six tasks. Thus, LFIRL's consistently strong results indicate that the sequential recovery procedure preserves both the linear variation and relative ordering of the ground-truth reward across a range of environments.

\subsection{Ablation Study}\label{sec:ablation}
\begin{wrapfigure}{r}{0.5\textwidth}
\vspace{-3.5em}
\begin{minipage}[t]{0.5\textwidth}
    \centering
    \includegraphics[height=3.3cm]{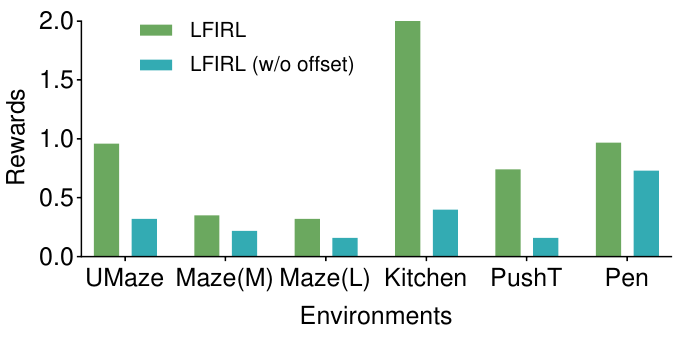}
    \captionof{figure}{\small Ablation results on PointMaze, Kitchen, Pen, and Push-T. We compare LFIRL with its variant that removes the state-dependent offset $b(\mathbf{s})$.}
    \label{fig:ablation}
    \vspace{-3em}
\end{minipage}
\end{wrapfigure}
We conduct an ablation study to evaluate the contribution of the state-dependent offset $b(\mathbf{s})$ in LFIRL. The results are shown in Fig.~\ref{fig:ablation}. Overall, removing $b(\mathbf{s})$ consistently degrades performance across the evaluated tasks, showing that inferring the offset and calibrating soft values are important for recovering a well-aligned value structure. This result supports our analysis in Sec.~\ref{sec:offset}: calibrating this offset is necessary for obtaining a reliable downstream reward.

\subsection{Reduced-Data Experiments}\label{sec:fewer_demo}

We further study how LFIRL behaves when the number of expert demonstrations is reduced. The results are shown in Fig.~\ref{fig:fewer_demo}, where we use the full dataset, one-half of the demonstrations, and one-quarter of the demonstrations. Overall, the performance of LFIRL decreases as the number of demonstrations becomes smaller. This trend is consistent with the design of our method: the first stage recovers the Q function from local supervision around expert actions, and the number of available expert state-action anchors directly affects how accurately this action-value structure can be reconstructed. When the demonstration set is reduced, anchor coverage becomes sparser, which weakens the quality of Q recovery and subsequently affects downstream value and reward recovery. Nevertheless, LFIRL remains competitive across the reduced-data settings.

\begin{figure}[tb]
    \centering
    \includegraphics[width=\columnwidth]{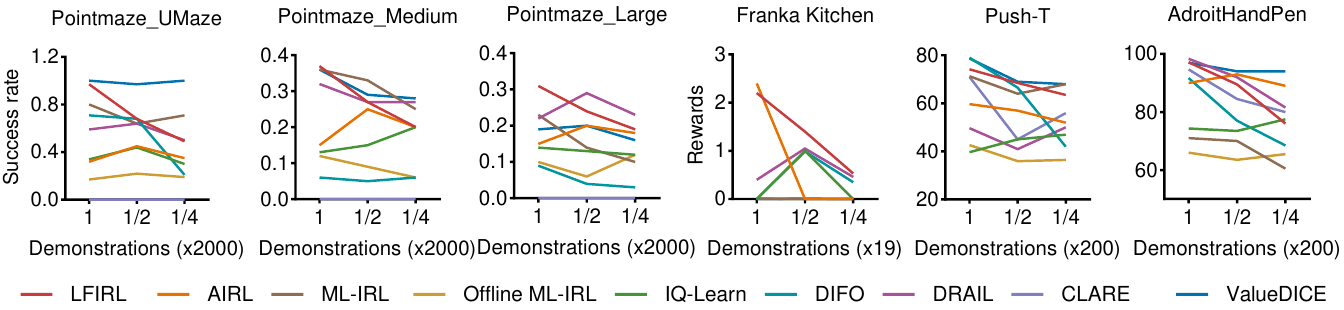}
    \caption{\small
    Reduced-data experiments with the full, one-half, and one-quarter of the demonstrations. We report success rate on PointMaze, average completed subtasks on Franka Kitchen, and reward on Push-T and Pen.
    }
    \label{fig:fewer_demo}
\end{figure}

\section{Discussion}

%\textbf{Demonstrations used in value and reward recovery.} The action value recovery relies on expert demonstrations because the score function induced by diffusion policy provides gradient supervision of expert actions. After $Q$-function has been recovered, however, the recovery of value function only needs to query the frozen $Q$-function to calibrate the optimization target, and the recovery of the reward function only needs transitions to form the Bellman function from the frozen $Q$ and $V$. Therefore, these two stages can use any pre-collected transition sets for sufficient coverage, including expert, random, or mixed trajectories, and maintain sufficient reward coverage and low training time cost at the same time.

\textbf{Source of the time efficiency gain.} The efficiency of LFIRL mainly comes from its linear training pipeline, not merely from using offline training framework. Offline IRL methods can still be slow if they keep a reward-policy loop, since reward updates repeatedly depend on policy or value optimization under the current reward, no matter if the policy interacts with the environment. Even methods that reduce this cost, such as IQ-Learn \citep{garg2021iq}, may avoid training an explicit policy in discrete action spaces but still rely on value or policy quantities induced by the current reward-related objective. LFIRL removes this repeated coupling with each stage using the frozen output of previous stages, and this is why LFIRL remains much faster than offline IRL methods in Tab.~\ref{tab:running_time}.

\textbf{Choice of generative models.} In fact, LFIRL uses a conditional {\em Denoising Diffusion Probabilistic Model} \citep{ho2020denoising} (DDPM)-based  diffusion policy to provide the derivative signal for Q function recovery, following common practice in diffusion-based imitation learning methods. In principle, any state-conditioned generative model that can provide a reliable score, or an action-gradient signal, can be used as the source of derivative supervision in Stage I. In Appendix~\ref{app:fm_score}, we instantiate a version based on {\em Flow Matching} \citep{lipmanflow}, another popular generative model paradigm. However, Appendix~\ref{app:ddpm_flow_matching} shows that, under the same number of training epochs, this variant gives weaker reward recovery. One possible reason is that the two pretraining objectives may have different convergence behavior under the same epoch budget. Conditional DDPM training directly optimizes denoising over multiple noise levels, which provides a strong and stable supervision signal for action refinement around expert demonstrations. Flow Matching instead learns a continuous transport vector field, whose action-generation quality can be more sensitive to integration accuracy, schedule design, and training budget. As a result, the pretrained diffusion policy is less accurate for giving the score signal used for Q function recovery, which then propagates to the final reward.

\section{Conclusion}\label{sec:conclusion}

We propose LFIRL, a novel staged IRL framework that removes the conventional reward-policy optimization loop by turning reward recovery into a sequence of value-recovery problems. Concretely, LFIRL first recovers an optimal Q function from diffusion-policy score signals and estimates the soft value function via Gumbel-regression-style value fitting. It then infers a state-dependent offset to calibrate the recovered soft values and re-estimates the value function using the calibrated Q function. Finally, it recovers the reward function by enforcing Bellman consistency. Experiments on PointMaze, Franka Kitchen, Adroit Hand Pen, and Push-T show that LFIRL matches or surpasses strong baselines in reward recovery while substantially reducing training time, achieving 2-3x efficiency gains in most settings. \noindent{\bf Limitations.} The quality of the recovered Q function depends on sufficient local anchors around expert actions. When anchor coverage is sparse, the method may require more demonstrations to reconstruct the action-value structure accurately. \noindent{\bf Future work} may explore broader classes of generative models beyond DDPM and Flow Matching for providing action-gradient supervision, and develop more sample-efficient ways to recover reliable value and reward information from limited expert demonstrations.

\newpage
\bibliographystyle{plainnat}
\medskip
\begin{small}
\bibliography{reference}
\end{small}

\newpage
%%%%%%%%%%%%%%%%%%%%%%%%%%%%%%%%%%%%%%%%%%%%%%%%%%%%%%%%%%%%

\appendix

\begin{appendices}

\section{Detailed Derivations}\label{app:derivations}

\subsection{Derivation of Diffusion Scores and Their Relation to Soft Q-Gradients}
\label{app:diffusion_score_q}

In this appendix, we derive the identities underlying Section~\ref{sec:pre}. We use \(t\) for the MDP time index and \(k\) for the diffusion step. For a fixed state \(\mathbf{s}_t\), let \(\mathbf{a}_t^0 \sim \pi_E(\cdot\mid \mathbf{s}_t)\) denote a clean expert action, and let \(\mathbf{a}_t^k\) denote its noisy version at diffusion step \(k\).

The DDPM forward process is defined by
\[
q(\mathbf{a}_t^k \mid \mathbf{a}_t^{k-1})
=
\mathcal{N}\!\left(
\mathbf{a}_t^k;
\sqrt{1-\beta_k}\,\mathbf{a}_t^{k-1},
\beta_k I
\right),
\qquad
\alpha_k \coloneqq 1-\beta_k,
\qquad
\bar{\alpha}_k \coloneqq \prod_{j=1}^{k}\alpha_j .
\]
This gives the closed-form marginal
\[
q(\mathbf{a}_t^k \mid \mathbf{a}_t^0)
=
\mathcal{N}\!\left(
\mathbf{a}_t^k;
\sqrt{\bar{\alpha}_k}\,\mathbf{a}_t^0,
(1-\bar{\alpha}_k)I
\right),
\]
or equivalently,
\begin{equation}
\label{eq:app_diff_add_noise}
\mathbf{a}_t^k
=
\sqrt{\bar{\alpha}_k}\,\mathbf{a}_t^0
+
\sqrt{1-\bar{\alpha}_k}\,\boldsymbol{\epsilon},
\qquad
\boldsymbol{\epsilon}\sim \mathcal{N}(0,I).
\end{equation}
This is the same forward noising relation used in Eq.~\eqref{eq:diff_add_noise}.

For each fixed \((\mathbf{s}_t,k)\), let \(p_k(\mathbf{a}\mid \mathbf{s}_t)\) denote the conditional density of noisy actions at diffusion step \(k\), induced by sampling \(\mathbf{a}_t^0\sim\pi_E(\cdot\mid\mathbf{s}_t)\) and then applying Eq.~\eqref{eq:app_diff_add_noise}. The score of this noisy action density is
\[
\nabla_{\mathbf{a}}\log p_k(\mathbf{a}\mid\mathbf{s}_t).
\]
The diffusion policy is trained by the noise-prediction objective
\[
\mathcal{L}_{\mathrm{diff}}({\bm\phi})
=
\mathbb{E}_{\substack{
(\mathbf{s}_t,\mathbf{a}_t^0)\sim \rho^{\pi_E}\\
k\sim \mathrm{Unif}\{1,\dots,K\},\,
\boldsymbol{\epsilon}\sim \mathcal{N}(0,I)
}}
\left\|
\boldsymbol{\epsilon}
-
\boldsymbol{\epsilon}_{\bm\phi}
\left(
\mathbf{a}_t^k,
\mathbf{s}_t,
k
\right)
\right\|_2^2 .
\]
For a fixed \((\mathbf{a}_t^k,\mathbf{s}_t,k)\), the pointwise minimizer of this MSE objective is
\[
\boldsymbol{\epsilon}_{\bm\phi}^*(\mathbf{a}_t^k,\mathbf{s}_t,k)
=
\mathbb{E}\!\left[
\boldsymbol{\epsilon}
\mid
\mathbf{a}_t^k,\mathbf{s}_t,k
\right].
\]
Using the Gaussian corruption in Eq.~\eqref{eq:app_diff_add_noise}, the noisy-action score satisfies
\[
\nabla_{\mathbf{a}}\log p_k(\mathbf{a}\mid\mathbf{s}_t)
\big|_{\mathbf{a}=\mathbf{a}_t^k}
=
-
\frac{1}{\sqrt{1-\bar{\alpha}_k}}
\,
\mathbb{E}\!\left[
\boldsymbol{\epsilon}
\mid
\mathbf{a}_t^k,\mathbf{s}_t,k
\right].
\]
Therefore, the learned noise predictor induces the score estimator
\begin{equation}
\label{eq:app_ddpm_score}
g_{\bm\phi}(\mathbf{a}_t^k,\mathbf{s}_t,k)
\coloneqq
-
\frac{1}{\sqrt{1-\bar{\alpha}_k}}
\boldsymbol{\epsilon}_{\bm\phi}(\mathbf{a}_t^k,\mathbf{s}_t,k)
\approx
\nabla_{\mathbf{a}}\log p_k(\mathbf{a}\mid\mathbf{s}_t)
\big|_{\mathbf{a}=\mathbf{a}_t^k}.
\end{equation}
In the idealized limit of exact denoising, the approximation in Eq.~\eqref{eq:app_ddpm_score} becomes exact for the noisy conditional density \(p_k(\mathbf{a}\mid\mathbf{s}_t)\).

For small diffusion noise, \(\mathbf{a}_t^k\) is close to the clean action \(\mathbf{a}_t^0\), and the noisy-action score approximates the clean expert-policy score in a local neighborhood of the expert action:
\[
g_{\bm\phi}(\mathbf{a}_t^k,\mathbf{s}_t,k)
\approx
\nabla_{\mathbf{a}}\log \pi_E(\mathbf{a}\mid\mathbf{s}_t)
\big|_{\mathbf{a}=\mathbf{a}_t^k}.
\]
If the expert policy is soft-optimal, \(\pi_E=\pi^*\). Under the reference-policy MaxEnt formulation in Eq.~\eqref{eq:boltzmann_mu}, we have
\[
\pi^*(\mathbf{a}\mid\mathbf{s})
=
\mu(\mathbf{a}\mid\mathbf{s})
\exp\!\left(
{\textstyle \frac{1}{\varepsilon}}
\big(
Q^*(\mathbf{s},\mathbf{a})-V^*(\mathbf{s})
\big)
\right).
\]
Taking the action gradient gives
\[
\nabla_{\mathbf{a}}\log \pi^*(\mathbf{a}\mid\mathbf{s})
=
\nabla_{\mathbf{a}}\log \mu(\mathbf{a}\mid\mathbf{s})
+
{\textstyle \frac{1}{\varepsilon}}
\nabla_{\mathbf{a}}Q^*(\mathbf{s},\mathbf{a}),
\]
because \(V^*(\mathbf{s})\) does not depend on \(\mathbf{a}\). Equivalently,
\[
\nabla_{\mathbf{a}}\log
{\textstyle \frac{\pi^*(\mathbf{a}\mid\mathbf{s})}{\mu(\mathbf{a}\mid\mathbf{s})}}
=
{\textstyle \frac{1}{\varepsilon}}
\nabla_{\mathbf{a}}Q^*(\mathbf{s},\mathbf{a}),
\]
which is Eq.~\eqref{eq:q_derivative}.

Combining the low-noise score approximation with Eq.~\eqref{eq:q_derivative}, we obtain
\begin{equation}
\label{eq:app_diffusion_q_connection}
g_{\bm\phi}(\mathbf{a}_t^k,\mathbf{s}_t,k)
-
\nabla_{\mathbf{a}}\log \mu(\mathbf{a}\mid\mathbf{s}_t)
\big|_{\mathbf{a}=\mathbf{a}_t^k}
\approx
{\textstyle \frac{1}{\varepsilon}}
\nabla_{\mathbf{a}}Q^*(\mathbf{s}_t,\mathbf{a})
\big|_{\mathbf{a}=\mathbf{a}_t^k}.
\end{equation}
This is the identity used in Eq.~\eqref{eq:q_stage1_loss}: the diffusion-policy score, after subtracting the reference-policy score, provides the action-gradient supervision for recovering the soft \(Q\)-function.

\subsection{Derivations for Stage I: Q-Recovery and Initial Value Fitting}
\label{app:q_recovery_derivations}

In this subsection, we justify the first-stage recovery of the uncalibrated soft $Q$-function and the initial value function. The key point is that matching action derivatives recovers the action-dependent geometry of $Q^*$, but cannot determine its state-dependent offset.

We next justify the Gumbel-regression objective used to obtain the initial value estimate from the frozen $\hat Q_{\bm\omega}$. For a fixed state $\mathbf{s}$ and a fixed Q function $\widetilde Q$, define
\[
\ell_V(v;\mathbf{s})
:=
\mathbb{E}_{\mathbf{a}\sim\mu(\cdot\mid\mathbf{s})}
\left[
\exp\!\left(\frac{\widetilde Q(\mathbf{s},\mathbf{a})-v}{\varepsilon}\right)
-
\frac{\widetilde Q(\mathbf{s},\mathbf{a})-v}{\varepsilon}
-1
\right].
\]
Differentiating with respect to $v$ gives
\[
\frac{\partial}{\partial v}\ell_V(v;\mathbf{s})
=
-\frac{1}{\varepsilon}
\mathbb{E}_{\mathbf{a}\sim\mu(\cdot\mid\mathbf{s})}
\left[
\exp\!\left(\frac{\widetilde Q(\mathbf{s},\mathbf{a})-v}{\varepsilon}\right)
\right]
+
\frac{1}{\varepsilon}.
\]
The first-order condition is therefore
\[
\mathbb{E}_{\mathbf{a}\sim\mu(\cdot\mid\mathbf{s})}
\left[
\exp\!\left(\frac{\widetilde Q(\mathbf{s},\mathbf{a})-v}{\varepsilon}\right)
\right]
=1.
\]
Rearranging gives
\[
v
=
\varepsilon
\log
\mathbb{E}_{\mathbf{a}\sim\mu(\cdot\mid\mathbf{s})}
\left[
\exp\!\left(\frac{1}{\varepsilon}\widetilde Q(\mathbf{s},\mathbf{a})\right)
\right].
\]
Moreover,
\[
\frac{\partial^2}{\partial v^2}\ell_V(v;\mathbf{s})
=
\frac{1}{\varepsilon^2}
\mathbb{E}_{\mathbf{a}\sim\mu(\cdot\mid\mathbf{s})}
\left[
\exp\!\left(\frac{\widetilde Q(\mathbf{s},\mathbf{a})-v}{\varepsilon}\right)
\right]
>0,
\]
so the minimizer is unique.

In Stage I, we instantiate $\widetilde Q$ as the frozen action value function $\hat Q_{\bm\omega}$. Replacing the expectation over $\mu(\cdot\mid\mathbf{s})$ by actions gives
\[
\mathcal{L}_{V}({\bm\chi})
=
\mathbb{E}_{\mathbf{s},\mathbf{a}\sim\mu(\cdot|\mathbf{s})}
\left[
\exp(z)-z-1
\right],
\qquad
z
=
\frac{\hat Q_{\bm\omega}(\mathbf{s}_t,\mathbf{a}_t)-\hat V_{\bm\chi}(\mathbf{s}_t)}{\varepsilon}.
\]
This is Eq.~\eqref{eq:v_loss_main_refined}. Since $\hat Q_{\bm\omega}$ is fixed in this step, gradients are taken only with respect to ${\bm\chi}$. The resulting value network $\hat V_{\bm\chi}$ is the initial value estimate used in Stage II for offset calibration.

\subsection{Derivations for Stage II: Offset Calibration and Value Re-estimation}
\label{app:offset_v_derivations}

\begin{proof}[Proof of Proposition~\ref{prop:q_up_to_b}]
Fix any state $\mathbf{s}$ and a connected action neighborhood on which the low-noise diffusion-score supervision is reliable. If the derivative-matching term in Eq.~\eqref{eq:q_stage1_loss} is minimized exactly, then for every initial action point $\mathbf{a}$ in this neighborhood,
\[
\frac{1}{\varepsilon}
\nabla_{\mathbf{a}}\hat Q_{\bm\omega}(\mathbf{s},\mathbf{a})
=
g_\phi(\mathbf{a}^k,\mathbf{s},k),
\]
where $\mathbf{a}^k$ denotes the low-noise corrupted version of $\mathbf{a}$. By Eq.~\eqref{eq:diffusion_q_connection}, under the idealized assumptions that $\pi_E=\pi^*$ and the diffusion score is exact,
\[
g_\phi(\mathbf{a}^k,\mathbf{s},k)
=
\frac{1}{\varepsilon}
\nabla_{\mathbf{a}}Q^*(\mathbf{s},\mathbf{a}).
\]
Thus,
\[
\nabla_{\mathbf{a}}\hat Q_{\bm\omega}(\mathbf{s},\mathbf{a})
=
\nabla_{\mathbf{a}}Q^*(\mathbf{s},\mathbf{a}).
\]
Since the gradients match only with respect to the action variable, the two functions can differ by a term that depends on the state but not on the action. Therefore, there exists a state-dependent function $b^*:\mathcal{S}\to\mathbb{R}$ such that
\[
Q^*(\mathbf{s},\mathbf{a})
=
\hat Q_{\bm\omega}(\mathbf{s},\mathbf{a})+b^*(\mathbf{s})
\]
on the sampled action neighborhood. This proves Proposition~\ref{prop:q_up_to_b}.
\end{proof}

\section{Flow-Matching Version of Diffusion Policy}
\label{app:flow_matching_policy}

\subsection{Flow-Matching Policy and Score Estimation}
\label{app:fm_score}

The main method uses a DDPM-based diffusion policy to provide the score estimator
\(g_\phi(\mathbf{a}_t^k,\mathbf{s}_t,k)\). To keep the presentation concise, this subsection follows the simplified case that the reference policy \(\mu\) is uniform. In this case, \(\nabla_{\mathbf a}\log \mu(\mathbf a\mid\mathbf s)=0\) inside the support, so the generative-model score can be directly used as the action-gradient signal for Stage I. This is not the only possible generative pretraining choice. A Flow-Matching-based model can also be used as a conditional action generator, and under a simple Gaussian interpolation path it can provide a score estimator that plays the same role in Stage I. This section explains this replacement.

Flow Matching trains a time-dependent vector field that transports samples from a simple noise distribution to the data distribution \citep{lipmanflow}. In the policy setting, the data distribution is the expert action distribution conditioned on the current state. Let \(u\in[0,1]\) denote the flow time, where \(u=0\) corresponds to Gaussian noise and \(u=1\) corresponds to expert actions. Given an expert pair \((\mathbf{s}_t,\mathbf{a}_t)\sim\mathcal D_E\) and Gaussian noise \(\boldsymbol{\epsilon}\sim\mathcal N(0,I)\), the interpolation is defined as:
\begin{equation}
\label{eq:flow_matching_interpolation}
\mathbf{a}_t^u
=
(1-u)\boldsymbol{\epsilon}
+
u\mathbf{a}_t .
\end{equation}
Here, \(\mathbf{a}_t^u\) is an intermediate noisy action along the path from noise to the expert action. The derivative of this path with respect to \(u\) is
\[
\frac{d}{du}\mathbf{a}_t^u
=
\mathbf{a}_t-\boldsymbol{\epsilon}.
\]
A Flow-Matching policy therefore learns a state-conditioned velocity field
\(v_\phi(\mathbf{a}_t^u,\mathbf{s}_t,u)\) by the regression objective
\begin{equation}
\label{eq:flow_matching_policy_loss}
\mathcal L_{\mathrm{FM}}(\phi)
:=
\mathbb E_{\substack{
(\mathbf{s}_t,\mathbf{a}_t)\sim\mathcal D_E\\
u\sim \mathrm{Unif}(0,1),\,
\boldsymbol{\epsilon}\sim\mathcal N(0,I)
}}
\left[
\left\|
v_\phi(\mathbf{a}_t^u,\mathbf{s}_t,u)
-
(\mathbf{a}_t-\boldsymbol{\epsilon})
\right\|_2^2
\right].
\end{equation}
After training, action generation can be performed by starting from Gaussian noise and integrating
\[
\frac{d}{du}\mathbf{a}^u
=
v_\phi(\mathbf{a}^u,\mathbf{s}_t,u)
\]
from \(u=0\) to \(u=1\). The resulting endpoint is used as an action sampled from the learned policy.

For LFIRL, we need not only a sampling procedure, but also a score signal for action-gradient matching. Let \(p_u(\mathbf{a}\mid\mathbf{s}_t)\) denote the conditional density of the intermediate action \(\mathbf{a}_t^u\) at flow time \(u\). Its score is
\[
\nabla_{\mathbf{a}}\log p_u(\mathbf{a}\mid\mathbf{s}_t),
\]
which indicates how an intermediate action should be adjusted to move toward higher-probability expert actions under state \(\mathbf{s}_t\).

Under the interpolation defined in Eq.~\eqref{eq:flow_matching_interpolation}, this score can be estimated from the learned velocity field. %If the learned velocity is exact, then at a fixed \((\mathbf{a}_t^u,\mathbf{s}_t,u)\) it estimates the conditional mean of the path derivative,
% \[
% \mathbf{a}_t^u
% =
% (1-u)\boldsymbol{\epsilon}
% +
% u\mathbf{a}_t .
% \]
If the learned velocity is exact, then at a fixed \((\mathbf{a}_t^u,\mathbf{s}_t,u)\) it estimates the conditional mean of the path derivative,
\[
v_\phi(\mathbf{a}_t^u,\mathbf{s}_t,u)
\approx
\mathbb E[\mathbf{a}_t-\boldsymbol{\epsilon}
\mid
\mathbf{a}_t^u,\mathbf{s}_t,u].
\]
Combining this relation with the interpolation equation gives
\[
\mathbb E[\boldsymbol{\epsilon}
\mid
\mathbf{a}_t^u,\mathbf{s}_t,u]
\approx
\mathbf{a}_t^u
-
u\,v_\phi(\mathbf{a}_t^u,\mathbf{s}_t,u).
\]
For the same Gaussian path, the score of \(p_u(\mathbf{a}\mid\mathbf{s}_t)\) can be written in terms of this conditional noise estimate. Thus, for \(u<1\), we obtain the Flow-Matching score estimator
\begin{equation}
\label{eq:fm_score_estimator}
g_\phi^{\mathrm{FM}}(\mathbf{a}_t^u,\mathbf{s}_t,u)
:=
-
\frac{
\mathbf{a}_t^u
-
u\,v_\phi(\mathbf{a}_t^u,\mathbf{s}_t,u)
}{
1-u
}
\approx
\nabla_{\mathbf{a}}
\log p_u(\mathbf{a}\mid\mathbf{s}_t)
\big|_{\mathbf{a}=\mathbf{a}_t^u}.
\end{equation}
The expression is used for \(u<1\); at \(u=1\), the path reaches the expert action endpoint and the denominator vanishes. In practice, the analogue of the low-noise DDPM regime is to use flow times close to, but not exactly equal to, \(1\).

This gives the connection needed by LFIRL. When \(u\) is close to \(1\), the intermediate action \(\mathbf{a}_t^u\) is close to the initial expert action. Therefore, \(g_\phi^{\mathrm{FM}}(\mathbf{a}_t^u,\mathbf{s}_t,u)\) estimates the local score of the expert action distribution around that action. %Under the uniform simplification and \(\pi_E=\pi^*\), Eq.~\eqref{eq:q_derivative} gives
% \[
% \nabla_{\mathbf{a}}\log \pi_E(\mathbf{a}\mid\mathbf{s}_t)
% =
% {\textstyle \frac{1}{\varepsilon}}
% \nabla_{\mathbf{a}}Q^*(\mathbf{s}_t,\mathbf{a}).
% \]
Thus, in the low-noise regime, Eq.~\eqref{eq:q_derivative} gives
\begin{equation}
\label{eq:fm_q_connection}
g_\phi^{\mathrm{FM}}(\mathbf{a}_t^u,\mathbf{s}_t,u)
\approx
{\textstyle \frac{1}{\varepsilon}}
\nabla_{\mathbf{a}}Q^*(\mathbf{s}_t,\mathbf{a})
\big|_{\mathbf{a}=\mathbf{a}_t^u}.
\end{equation}

Therefore, a Flow-Matching-based diffusion policy can be used in Stage I by replacing the DDPM score estimator \(g_\phi(\mathbf{a}_t^k,\mathbf{s}_t,k)\) in Eq.~\eqref{eq:q_stage1_loss} with \(g_\phi^{\mathrm{FM}}(\mathbf{a}_t^u,\mathbf{s}_t,u)\).

\begin{figure}[ht]
    \centering
    \includegraphics[width=0.6\columnwidth]{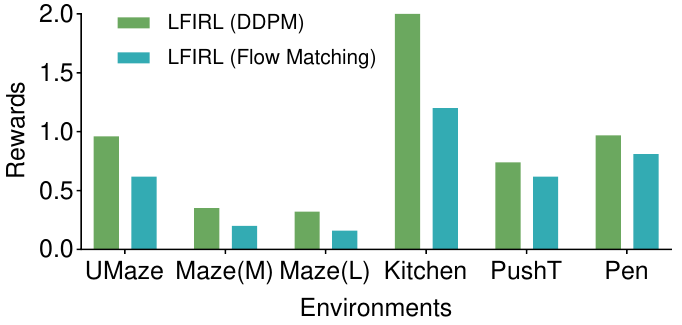}
    \caption{\small
    Additional ablation comparing DDPM-based and Flow-Matching-based pretraining in LFIRL. The two variants use the same subsequent sequential reward-recovery pipeline, but differ in how the pretrained generative policy is obtained. Higher is better.
    }
    \label{fig:flow_matching_ablation}
\end{figure}

\subsection{Additional Results: DDPM vs. Flow Matching Pretraining}
\label{app:ddpm_flow_matching}

In the main experiments, LFIRL uses a DDPM-based diffusion policy to provide the score signal for recovering the soft $Q$-function. We also evaluate an alternative implementation in which the diffusion-policy pretraining stage is replaced by a Flow Matching version, while keeping the remaining LFIRL pipeline unchanged. This comparison isolates the effect of the pretrained generative policy used to provide local action-gradient supervision.

As shown in Fig.~\ref{fig:flow_matching_ablation}, under the same number of pretraining epochs, the Flow-Matching-based variant generally performs worse than the DDPM-based variant. This suggests that, in our setting, DDPM pretraining provides a more effective score signal for the subsequent $Q$-recovery stage. Since LFIRL relies on the pretrained generative policy to supply local action-gradient information, a weaker pretraining stage can directly reduce the quality of the recovered $Q$-function and subsequently degrade the final recovered reward.

\section{Sensitivity to the Reference Policy}
\label{app:reference_policy_sensitivity}

The reference policy $\mu$ specifies the behavioral baseline relative to which the expert policy is interpreted. Under the reference-policy MaxEnt formulation,
\begin{equation}
\label{eq:app_reference_policy_gradient}
\nabla_{\mathbf a}\log\pi_E(\mathbf a\mid\mathbf s)
=
\nabla_{\mathbf a}\log\mu(\mathbf a\mid\mathbf s)
+
\frac{1}{\varepsilon}
\nabla_{\mathbf a}Q^*(\mathbf s,\mathbf a).
\end{equation}
We examine how the choice of $\mu$ affects LFIRL by comparing a uniform reference policy with a state-conditioned Gaussian reference policy.

The two variants use the same expert demonstrations and LFIRL training configuration. The uniform variant uses a uniform distribution over the bounded action space, for which
$\nabla_{\mathbf a}\log\mu(\mathbf a\mid\mathbf s)=0$
in the interior of its support. For the Gaussian variant, we initialize an SAC actor before LFIRL training and freeze it throughout the reward-recovery pipeline. The actor outputs the mean and standard deviation of a state-conditioned action distribution and is not fitted to the expert demonstrations. At each noisy action $\mathbf a_t^k$, we evaluate the score of this frozen reference policy and use
\[
g_{\bm\phi}(\mathbf a_t^k,\mathbf s_t,k)
-
\nabla_{\mathbf a}\log\mu(\mathbf a\mid\mathbf s_t)
\big|_{\mathbf a=\mathbf a_t^k}
\]
as the target for
$\varepsilon^{-1}\nabla_{\mathbf a}\hat Q_{\bm\omega}(\mathbf s_t,\mathbf a_t^k)$.
Actions used for soft-value estimation are also sampled from the same frozen reference policy.

We evaluate these variants on MuJoCo tasks, where the expert demonstrations can be generated by stochastic sampling from pretrained SAC experts. This allows the reference-policy comparison to be conducted with a specified source of expert actions; the generating policy distributions of the robotic-manipulation datasets cannot be verified in the same way. Table~\ref{tab:reference_policy_sensitivity} reports the resulting episodic returns.

\begin{table}[ht]
    \centering
    \small
    \setlength{\tabcolsep}{5pt}
    \renewcommand{\arraystretch}{1.08}
    \caption{\small
        Sensitivity to the reference policy on MuJoCo tasks. Results are episodic returns, reported as mean $\pm$ standard deviation. The expert return is included for context.
    }
    \label{tab:reference_policy_sensitivity}
    \vspace{0.5em}
    \begin{tabular}{lccc}
        \toprule
        \textbf{Environment}
        & \textbf{LFIRL-Uniform}
        & \textbf{LFIRL-Gaussian}
        & \textbf{Expert} \\
        \midrule
        Swimmer-v5
        & 284.56$\pm$27.64
        & 208.42$\pm$7.25
        & 315.55$\pm$1.38 \\
        Walker2d-v5
        & 4933.41$\pm$52.79
        & 4722.40$\pm$230.42
        & 5861.18$\pm$73.99 \\
        Hopper-v5
        & 3413.50$\pm$210.23
        & 3868.04$\pm$79.58
        & 4098.17$\pm$247.70 \\
        \bottomrule
    \end{tabular}
\end{table}

The choice of reference policy affects performance, and the direction of the effect depends on the task. The uniform variant performs better on Swimmer and Walker2d, whereas the Gaussian variant performs better on Hopper. LFIRL therefore is not invariant to $\mu$: changing the behavioral baseline changes the reward interpretation and can alter the recovered reward. Nevertheless, the Gaussian-reference variant remains effective, showing that the recovery procedure is not restricted to the uniform-reference simplification.

This experiment assumes that the reference policy is specified. It does not infer $\mu$ from expert demonstrations. In general, demonstrations alone do not uniquely separate an unknown reference policy from the reward without additional assumptions; jointly inferring the two is beyond the scope of this work.

\section{Sensitivity to Stage-I Hyperparameters}
\label{app:stage1_sensitivity}

LFIRL has several sequential recovery stages, but their hyperparameters need not have the same effect on performance. The value-recovery stage uses a Gumbel-regression-style objective, while the final reward-recovery stage enforces Bellman consistency. Here, we focus on three choices that directly affect the local $Q$-recovery signal in Stage I: the magnitude of the diffusion noise, the value-anchoring coefficient $\lambda$, and the intended margin $\xi$ in value anchoring.

We vary one parameter at a time while holding the other two at their default settings: the low-noise schedule, $\lambda=1$, and $\xi=1$. The noise settings scale the coefficient $\sqrt{1-\bar{\alpha}_k}$ in the forward noising process by $1\times$, $4\times$, and $8\times$, respectively. Table~\ref{tab:stage1_sensitivity} reports PointMaze success rates under these settings.

\begin{table*}[t]
    \centering
    \small
    \setlength{\tabcolsep}{5pt}
    \renewcommand{\arraystretch}{1.08}
    \caption{\small
        Sensitivity of PointMaze success rate to Stage-I hyperparameters. Each group varies one parameter while the remaining parameters are held at their default values. Bold indicates the highest result within each parameter group for each environment.
    }
    \label{tab:stage1_sensitivity}
    \vspace{0.5em}
    \begin{tabular}{lccc ccc ccc}
        \toprule
        & \multicolumn{3}{c}{\textbf{Diffusion-noise magnitude}}
        & \multicolumn{3}{c}{\textbf{Value-anchoring coefficient}}
        & \multicolumn{3}{c}{\textbf{Intended margin}} \\
        \cmidrule(lr){2-4}
        \cmidrule(lr){5-7}
        \cmidrule(lr){8-10}
        \textbf{Environment}
        & Low 
        & Medium 
        & High 
        & $\lambda=0.5$
        & $\lambda=1$
        & $\lambda=2$
        & $\xi=0.5$
        & $\xi=1$
        & $\xi=2$ \\
        \midrule
        UMaze
        & \textbf{0.98} & 0.87 & 0.68
        & 0.92 & \textbf{0.98} & 0.97
        & 0.57 & \textbf{0.98} & 0.61 \\
        Medium
        & \textbf{0.39} & 0.35 & 0.25
        & 0.38 & \textbf{0.39} & 0.35
        & 0.24 & \textbf{0.39} & 0.26 \\
        Large
        & \textbf{0.31} & 0.29 & 0.25
        & 0.26 & 0.31 & \textbf{0.34}
        & 0.18 & \textbf{0.31} & 0.24 \\
        \bottomrule
    \end{tabular}
\end{table*}

\paragraph{Diffusion-noise magnitude.}
Performance decreases as the noise coefficient increases. Increasing
$\sqrt{1-\bar{\alpha}_k}$ moves noisy actions farther from the clean expert-action distribution, where the local approximation between the noisy-action score and the expert-policy score becomes less reliable. The decline is modest from low to medium noise on Medium and Large, but larger at high noise across all three environments. This supports using the low-noise schedule as the default for local $Q$-gradient supervision.

\paragraph{Value-anchoring coefficient $\lambda$.}
Performance is comparatively stable between $\lambda=1$ and $\lambda=2$. Increasing $\lambda$ to $2$ slightly improves Large while causing a small decrease on UMaze and Medium. Reducing it to $0.5$ weakens the anchoring signal and lowers performance, particularly on Large. Thus, the method is less sensitive to this coefficient once anchoring is sufficiently strong, with $\lambda=1$ providing a stable choice across the three environments.

\paragraph{Intended margin $\xi$.}
The intended margin has the strongest effect among the three tested parameters. Both $\xi=0.5$ and $\xi=2$ underperform $\xi=1$ on every environment. A small margin provides insufficient separation between expert actions and their perturbed neighbors. An excessively large margin demands larger local $Q$ differences and may conflict with gradient matching or suppress near-optimal actions. In these experiments, $\xi=1$ provides the most reliable balance across all three environments.

Overall, these results identify the diffusion-noise magnitude and the anchoring margin as the more consequential Stage-I settings. The value-anchoring coefficient is less sensitive once it is large enough to provide effective anchoring.

\section{Additional Experimental Details}\label{app:details}

\subsection{Implementation Details and Hyperparameter Setup}\label{app:setup}

The overall training procedure of LFIRL is given in Alg.~\ref{alg:loop_free_irl}. Key network architecture and hyperparameter settings for each environment are summarized in Tab.~\ref{tab:setup_LFIRL_maze} and Tab.~\ref{tab:setup_LFIRL_robot}.

\begin{table}[H]
    \centering
    \small
    \caption{\small Network architecture and hyperparameter setup for the PointMaze environments.}
    \label{tab:setup_LFIRL_maze}
    \begin{tabular}{lccc}
    \toprule
      & PointMaze\_UMaze-v3 & PointMaze\_Medium-v3 & PointMaze\_Large-v3 \\
    \midrule
      Expert demonstrations
        & 2000 & 2000 & 2000 \\
      Policy pretraining epochs
        & 60 & 40 & 40 \\
      Q/$b$/V/$r$ learning rate
        & 3e-5 & 3e-5 & 3e-5 \\
      Stage passes $(Q,b,V,r)$
        & 8/8/20/20 & 8/8/20/20 & 8/8/20/20 \\
      Q, $b$, V, $r$ hidden layers
        & 256, 256, 256, 256 & 256, 256, 256, 256 & 256, 256, 256, 256 \\
      IRL batch size
        & 256 & 256 & 256 \\
      Discount factor $\gamma$
        & 0.99 & 0.99 & 0.99 \\
      EMA coefficient $\tau$
        & 0.005 & 0.005 & 0.005 \\
    \bottomrule
    \end{tabular}
\end{table}

\begin{table}[H]
    \centering
    \small
    \caption{\small Network architecture and hyperparameter setup for the robotic manipulation environments.}
    \label{tab:setup_LFIRL_robot}
    \begin{tabular}{lccc}
    \toprule
      & FrankaKitchen-v1 & gym\_pusht/PushT-v0 & AdroitHandPen-v1 \\
    \midrule
      Expert demonstrations
        & 19 & 200 & 200 \\
      Policy pretraining epochs
        & 20 & 200 & 40 \\
      Q/$b$/V/$r$ learning rate
        & 3e-5 & 3e-4 & 3e-5 \\
      Stage passes $(Q,b,V,r)$
        & 8/8/20/20 & 40/40/100/100 & 40/40/100/100 \\
      Q, $b$, V, $r$ hidden layers
        & 256, 256, 256, 256 & 256, 256, 256, 256 & 256, 256, 256, 256 \\
      IRL batch size
        & 256 & 256 & 256 \\
      Discount factor $\gamma$
        & 0.99 & 0.99 & 0.99 \\
      EMA coefficient $\tau$
        & 0.005 & 0.005 & 0.005 \\
    \bottomrule
    \end{tabular}
\end{table}

\subsection{Expert Demonstrations Sources}\label{app:expert_demos}

The sources of offline trajectory datasets for experts are provided in Tab.~\ref{tab:experts}. In Maze and Kitchen tasks, we directly use the expert trajectories from the Minari Offline Reinforcement Learning datasets \citep{minari}. For the Push-T task, we use the expert trajectories from the dataset in Diffusion Policy \citep{chi2025diffusion}.

\begin{table}[H]
    \centering
    \small
    \caption{\small The sources of expert policies or demonstrations.}
    \label{tab:experts}
    \setlength{\tabcolsep}{6pt}
    \renewcommand{\arraystretch}{1.05}
    \begin{tabularx}{\linewidth}{lp{12cm}}
    \toprule
      Task   & Source \\
    \midrule
    UMaze           & \url{https://minari.farama.org/datasets/D4RL/pointmaze/umaze-v2/} \\
    Medium          & \url{https://minari.farama.org/datasets/D4RL/pointmaze/medium-v2/} \\
    Large           & \url{https://minari.farama.org/datasets/D4RL/pointmaze/large-v2/} \\
    Franka Kitchen  & \url{https://minari.farama.org/datasets/D4RL/kitchen/complete-v2/} \\
    Push-T          & \url{https://diffusion-policy.cs.columbia.edu/data/training/pusht.zip} \\
    Pen         & \url{https://minari.farama.org/datasets/D4RL/pen/expert-v2/} \\
    \bottomrule
    \end{tabularx}
\end{table}

\subsection{Hardware Information}\label{app:hardware}

Hardware specifications are provided in Tab.~\ref{tab:hardware}.

\begin{table}[H]
    \centering
    \caption{Hardware configuration used in experiments.}
    \label{tab:hardware}
    \begin{tabular}{ll}
    \toprule
      Hardware & Specifications \\
    \midrule
      CPU    & AMD EPYC 7713 64-Core Processor @ 2 GHz \\
      GPU    & NVIDIA A100-SXM4-80GB @ 1215 MHz \\
      Memory & 2 TB \\
    \bottomrule
    \end{tabular}
\end{table}

\end{appendices}

% \newpage
% \input{checklist.tex}

%%%%%%%%%%%%%%%%%%%%%%%%%%%%%%%%%%%%%%%%%%%%%%%%%%%%%%%%%%%%

\end{document}